\documentclass[twoside,11pt]{article}

 \usepackage[abbrvbib, preprint]{jmlr2e}

\usepackage{tikz} 
\usepackage{natbib} 
\usepackage{mathtools} 
\usepackage{booktabs} 

\usepackage{amsfonts}
\usepackage{siunitx} 
\usepackage{afterpage}
\usepackage{appendix}
\usepackage{amsmath}
\usepackage{mathtools}
\usepackage[ruled,vlined]{algorithm2e}
\usepackage{subcaption}

\newcommand\independent{\protect\mathpalette{\protect\independenT}{\perp}}
\def\independenT#1#2{\mathrel{\rlap{$#1#2$}\mkern2mu{#1#2}}}

\usepackage[noabbrev,capitalize]{cleveref}

\usepackage{ifthen}

\ShortHeadings{High-Dimensional Inference in Bayesian Networks}{Bayer et al.}
\firstpageno{1}

\begin{document}

\title{High-Dimensional Inference in Bayesian Networks}

\author{\name Fritz M. Bayer \email frbayer@ethz.ch \\
       	\addr ETH Zurich
       	\AND
       	\name Giusi Moffa \\
       	\addr University of Basel\\
       	\name Niko Beerenwinkel \\
       	\addr ETH Zurich\\
	   	\name Jack Kuipers \email jack.kuipers@bsse.ethz.ch \\
   		\addr ETH Zurich\\
}

\editor{}

\maketitle

\begin{abstract}
Inference of the marginal probability distribution is defined as the calculation of the probability of a subset of the variables and is relevant for handling missing data and hidden variables.
While inference of the marginal probability distribution is crucial for various problems in machine learning and statistics, its exact computation is generally not feasible for categorical variables in Bayesian networks due to the NP-hardness of this task. 
We develop a divide-and-conquer approach using the graphical properties of Bayesian networks to split the computation of the marginal probability distribution into sub-calculations of lower dimensionality, thus reducing the overall computational complexity. Exploiting this property, we present an efficient and scalable algorithm for calculating the marginal probability distribution for categorical variables. The novel method is compared against state-of-the-art approximate inference methods in a benchmarking study, where it displays superior performance. As an immediate application, we demonstrate how our method can be used to classify incomplete data against Bayesian networks and use this approach for identifying the cancer subtype of kidney cancer patient samples. 
\footnote{Code implementing our method and reproducible benchmarks are open-source and publicly available at \url{https://github.com/cbg-ethz/SubGroupSeparation}.}
\end{abstract}

\begin{keywords}
  Bayesian Networks, Probabilistic Graphical Models, Inference, Missing Data, Marginalization
\end{keywords}

\section{Introduction}\label{sec:intro}

Bayesian networks are a valuable tool for modelling the underlying dependencies between random variables \citep{pearl1995} and for discovering causal relations \citep{spirtes2000,maathuis2009}. By visualizing the probabilistic relationships in a compact graphical form, Bayesian networks are widely employed across various disciplines, such as psychology \citep{van2017, moffa2017}, medicine \citep{mclachlan2020}
and biology \citep{friedman2000,friedman2004}. As Bayesian networks are the most popular causal models and learning causal relations, rather than correlations, has increasingly gained popularity in machine learning and statistics in recent years \citep{scholkopf2019causality, luo2020causal}, various methods have been proposed to learn their structure \citep{heinze2018causal, constantinou2021, rios2021}. While recent advances permit Bayesian network sampling and structure learning to perform well for larger sizes \citep{kuipers2022efficient}, the task of inference of the marginal probability distribution of any observed subset of variables remains a fundamental problem in high dimensions for categorical variables. This task is essential in many statistical and scientific studies \citep{gelman1998}, in particular when dealing with missing data or hidden variables.

While inference of the marginal probability distribution is trivial for low dimensions, in general, the problem was proven to be NP-hard for categorical variables \citep{cooper1990}. That makes an exact calculation computationally unfeasible for the high-dimensional data that Bayesian networks are increasingly employed on \citep{kalisch2007, buhlmann2014}. Consequently, efforts are focused on finding an approximation of the marginal probability distribution for which two main branches of methods can be used: variational inference and importance sampling \citep{koller2009, murphy2012}. Variational inference allows for a fast calculation for categorical variables that is exact in Bayesian networks whose underlying directed acyclic graph (DAG) is a tree \citep{pearl1988} and otherwise approximate. However, the estimation is biased and not guaranteed to converge to the true value \citep{murphy1999}. In contrast, importance sampling techniques are unbiased, though they can take a long time to reach acceptable accuracy \citep{koller2009}. Since missing data is commonly handled by imputing the missing variables using the EM-algorithm in Bayesian networks \citep{scanagatta2018efficient, scutari2020bayesian, ruggieri2020hard}, previous methods have focused on inferring the probability distribution of a single variable given the evidence. In this work, we solve the more general problem of inferring the marginal probability distribution of multiple variables. 

We show that by exploiting the graphical properties of Bayesian networks, the inference of the marginal probability distribution can be split into sub-calculations of lower dimensionality. Exploiting this property by applying the optimal inference method on each sub-problem, we present an efficient algorithm for calculating the marginal probability distribution. 
Experimentally and theoretically, we answer the following questions in the positive:
\begin{itemize}
	\item Is the proposed method applicable for high-dimensional Bayesian networks?
	\item Does it perform favourably compared to state-of-the-art approximate inference methods?
	\item Is the proposed method unbiased (i.e.\ does it converge to the true value)? 
\end{itemize} 
Finally, as an application, we demonstrate how incomplete data can be classified against Bayesian network models by identifying the cancer subtype of kidney cancer patient samples.



\section{Preliminaries}

\subsection{Bayesian network notation}

A directed graph $\mathcal{G}=(V, E)$ consists of a set of nodes $V$ and directed edges $E \subseteq \{ ( i,j )  : i,j \in V , i \neq j \} $. It is said to be acyclic if there is no directed path that revisits the same node. The set of parent nodes for $i \in V$ is defined as ${pa(i) := \{j \in V : (j, i) \in E\} }$, which is the set of nodes that have a directed edge pointing into node $i$. 


A Bayesian network $(\mathcal{G}, P)$ comprises a directed acyclic graph $\mathcal{G}$ and associated local probability distributions $P=(P_i)_{i\in V}$, specifying the conditional probability distribution of each variable $X_i$ given its parents, $P_i=P\left(X_i \mid X_{pa(i)}\right)$. In the case of categorical variables, $P$ may simply be a conditional probability table (CPT). As a necessary condition for Bayesian networks, the DAG and its corresponding probability distribution $P$ need to satisfy the Markov properties \citep[section 2.2.4]{korb2010}. This enables the factorization of the probability of a complete observation
\begin{equation}
P(X_V)=\prod_{i \in V} P\big(X_i \mid X_{pa(i)}\big)
\end{equation}
where $P(X_V)= P(\{X_i\}_{i \in V})$. The marginal probability distribution of any subset of nodes $e\subseteq V$ can then be calculated by summing over the remaining variables ${V'=V\setminus e}$
\begin{equation}\label{equ:normConstDef}
P(X_e)=\sum_{X_{V'}} P(X_{V'},X_e)= \sum_{X_{V'}} \prod_{i\in V} P\big(X_i \mid X_{pa(i)}\big)
\end{equation}
where the subset of variables $X_e$ could, for example, represent a set of observed variables and will subsequently be referred to as the \textit{evidence variables}. By integrating out the unobserved variables, the marginal probability distribution $P(X_e)$ acts as a normalizing constant for the remaining variables ${X_{V'}}$
\begin{equation}
P(X_{V'}|X_e)=\frac{P(X_{V'},X_e)}{P(X_e)}
\end{equation}
and is frequently referred to as a normalising constant. Note that the marginal probability distribution marks the general case of the popular problem of obtaining the probability distribution of a single variable $P(X_i|X_e),\; i \in V$, frequently referred to as inference \citep{koller2009, murphy2012}. Implications of this generalization will be discussed in the following section.

\subsection{Marginalization and Message Passing}\label{sec:MessagePassing}

Message passing consists of a large family of algorithms that originate from the belief propagation algorithm as proposed by \cite{pearl1988}. For categorical variables, the most prominent ones include standard belief propagation for trees, loopy belief propagation and the junction tree algorithm \citep{koller2009}. The classical inference problem of calculating the probability distribution of a single variable given the evidence ${P(X_i \mid X_e)}$ can be efficiently solved by using belief propagation for trees; if the DAG is not a tree, an approximate solution can be found using loopy belief propagation \citep{murphy1999}. It is straightforward to obtain $P(X_i \mid X_e)$ from $P(X_i \mid X_{pa(i)},X_e)$ by summing over the parents. However, the inference of the marginal probability distribution requires calculation of the probability of multiple variables given the evidence $P(X_{V'}|X_e)$. This cannot generally be obtained from $P(X_i \mid X_{pa(i)},X_e)$ since
%
\begin{equation}
P(X_{V'} \mid X_e) \neq \prod_{i \in V'} P(X_i \mid X_{pa(i)}, X_e)
\end{equation}
under the condition that any pair of parents of the evidence variables $X_e$ are conditionally independent, i.e.\ ${X_i \independent  X_j}$ for any $i,j \in \{pa(k): k \in e \}$,  $i\neq j$. Thus, message passing for approximate inference is not ideal for inferring the marginal probability distribution, though it can be taken as a rough approximation, giving a good importance function for importance sampling \citep{yuan2006}. 

In contrast, the junction-tree algorithm yields the exact $P(X_{V'} \mid X_e)$ by combining multiple nodes into cliques and performing belief propagation on the cliques instead of on the individual nodes. With that, the junction-tree algorithm is efficient in inferring the marginal probability distribution in low-dimensional networks, but struggles with high-dimensional and dense networks as the computational complexity is exponential in the size of the largest clique. 

\subsection{Marginalization and Sampling}\label{sec:sampling}

Importance sampling is commonly applied when a target distribution $P'$ is intractable. By sampling from a tractable distribution $P_s$ that is cheap to compute, inference can be made about the target distribution. The challenge in importance sampling lies in the choice of the sampling distribution, which should be close to the target distribution and cheap to compute at the same time. 
For Gibbs sampling, the distribution $P_s$ would be the probability distribution of each variable given its Markov Blanket. 

Using importance weighting, the marginal probability distribution can computed from the samples from $P_s$ as follows
%
%
\begin{equation}
P(X_e)=\sum_{X_{V'}} P(X_{V'},X_e) = \sum_{X_{V'}} P_s(X_{V'}) \frac{P(X_{V'}, X_e)}{P_s(X_{V'})} 
= \mathbb{E}_{P_s(X_{V'})}\bigg[ \frac{P(X_{V'}, X_e)}{P_s(X_{V'})} \bigg]
\end{equation}
As in message passing, the problem of marginalization in importance sampling is different from the problem of calculating the probability distribution of a single variable given the evidence ${P(X_i \mid X_e)}$. Marginalization requires a good sampling distribution of all variables not in the evidence, whereas the calculation of ${P(X_i \mid X_e)}$ requires only a good sampling distribution of the respective variable $X_i$.

\subsection{Related Work}

While variational inference and importance sampling are usually the only option in high-dimensional settings, they can be rather slow or inaccurate in comparison to exact methods in lower dimensions. Thus, the ideal marginalization method largely depends on the dimensionality of the problem. 
Previous efforts in inferring the probability distribution of a single variable have focussed on finding hybrid methods that combine elements of importance sampling with exact or variational inference methods \citep{yuan2006, gogate2008, gogate2010, venugopal2013, li2017}. 
A prominent approach, known as collapsed particles \citep{koller2009}, Rao-Blackwellisation \citep{casella1996a} or cutset sampling \citep{bidyuk2006,bidyuk2007}, is to sample only a fraction of the variables, while performing exact inference on the others. While this approach improves the performance in some settings, it has two major drawbacks: Finding the ideal subset of variables from which to sample is an NP-hard problem in cutset sampling \citep{bidyuk2007}; and the computational complexity of the sapling process grows exponentially with each variable that is calculated by means of exact inference. In the following section, we show how both of these drawbacks can be overcome by introducing a novel divide-and-conquer formalism.

\cite{madsen1999} proposed a decomposition based on d-separation, reducing the computational complexity of the junction-tree algorithm. However their framework solves a more general problem and in the special case of marginalization, it only reduces the amount of passed messages in the junction-tree algorithm. 




\section{Subgroup Separation Scheme}


In this section, we show how the graphical properties of a Bayesian network can be exploited to reduce computational complexity of marginalization. Firstly, we show which nodes do not contribute to the marginal probability distribution and exclude them from further calculations. Secondly, by exploiting the concept of d-separation \citep{pearl1988}, we show how the variables in the graph can be split into subsets that are independent of each other in the marginalization. Thirdly, we develop a divide-and-conquer formalism that allows inferring each subset of variables independently; the reduced dimension of the subsets allows us to efficiently apply exact inference on parts of the calculation. 

\subsection{Reduction of Sampled Variables} \label{sec:reducedNet}

All variables that do not contribute to the marginal probability distribution should be neglected in further considerations. For this purpose, we define an \textit{irrelevant node} as a node $i$ whose descendants $de(i)$ do not contain the evidence nodes. Expanding on this definition, we introduce the \textit{relevant subgraph} as a DAG without the irrelevant nodes, similar to \citep{lin1997}.

\begin{definition}[Irrelevant Node]\label{def:superf}
	A node $i\in V$ in a DAG $\mathcal{G}=(V,E)$ over $X_V$ is irrelevant w.r.t.\ a set of nodes $e$ if $\Big(\{i\} \cup de(i) \Big) \cap e = \emptyset$.
\end{definition}

\begin{definition}[Relevant Subgraph]
	The relevant subgraph $\mathcal{G}'$ of a DAG $\mathcal{G}$ w.r.t.\ a set of nodes $e$ is the remaining graph after removal of all irrelevant nodes and their edges. 
\end{definition}

\noindent
Figure~\ref{fig:example} shows a DAG with evidence nodes marked in grey (a) and its corresponding relevant subgraph (b). As shown in Proposition~\ref{prop:relSub}, it is sufficient to consider the relevant subgraph $\mathcal{G}'$ for marginalization, since $P_{\mathcal{G}'}(X_e)=P_{\mathcal{G}}(X_e)$.

\begin{figure*}
	\centering
	\includegraphics[trim={0 2.5cm 0 3cm},clip,scale=0.38, page=2]{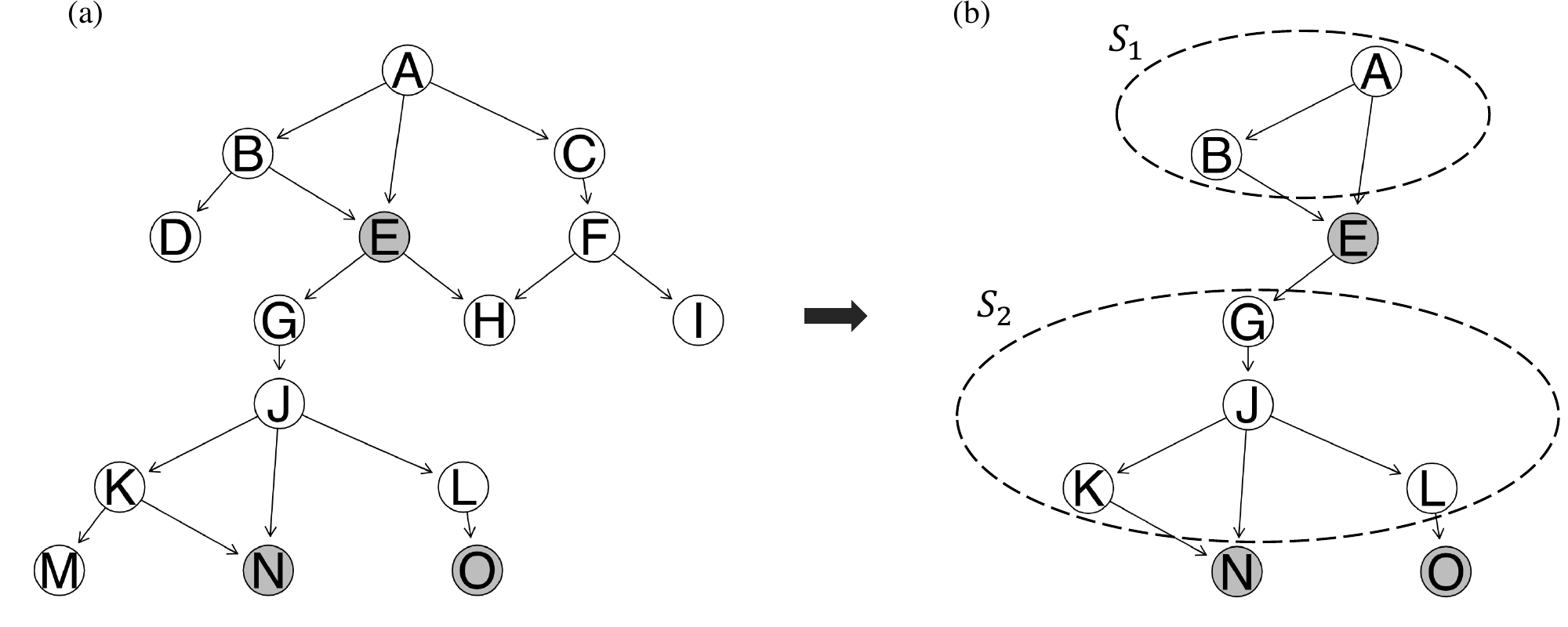}
	\caption{An example DAG with evidence nodes $e$ marked in grey (a) and the respective relevant subgraph (b). The relevant subgraph is split by the evidence into two conditionally independant subsets.}
	\label{fig:example}
\end{figure*}

\begin{proposition}[Marginalization over Relevant Subgraph] \label{prop:relSub}
	Let $\mathcal{G}'$ be the relevant subgraph of a DAG $\mathcal{G}$ w.r.t.\ a set of nodes $e$ and let $P_{\mathcal{G}'}$ and $P_{\mathcal{G}}$ be the respective probability distributions that satisfy the Markov properties. Then $P_{\mathcal{G}'}(X_e)=P_{\mathcal{G}}(X_e)$.
\end{proposition}

\noindent
A consequence of Proposition~\ref{prop:relSub} is that only a fraction of the variables need to be considered for exact or approximate inference, when calculating marginal probability distributions. 

\subsection{Separation into Subsets}

We further introduce the concept of \textit{conditionally independent subsets}, which describes that the evidence can separate the relevant subgraph into groups of conditionally independent variables.
\begin{definition}[Conditionally Independent Subset]\label{def:CIS}
	Let $U \subset V$. A set of variables $X_U$ is a conditionally independent subset w.r.t.\ a set of variables $X_e$, if 
	\begin{itemize}
		\item all variables in the subset are d-connected, i.e.\ $X_i$ is d-connected to $X_j$ w.r.t.\ $e$, ${\forall  i,j \in U}$, and
		\item all variables in the subset are d-separated from the remaining variables, i.e.\ $X_i$ is d-separated from $X_j$ w.r.t.\ $e$, ${\forall i \in U , j \in V \setminus \{U \cup e \}}$.
	\end{itemize}
\end{definition}
As an example, consider the DAG shown in Figure~\ref{fig:example}\,(a) with the evidence nodes $E$, $N$ and $O$ marked in grey. As illustrated in Figure~\ref{fig:example}\,(b), the variables can be split into the two conditionally independent subsets $S_1$ and $S_2$. 
Based on the conditions of Definition~\ref{def:CIS}, the variables $A$ and $B$ form a conditionally independent subset since they are d-connected and at the same time d-separated from all remaining variables ($G$, $J$, $K$ and $L$) by the evidence variable $E$. The same holds true for the variables of the subset $S_2$.



\begin{proposition}[Uniqueness]\label{prop:unique}
	The conditionally independent subsets $S_{\text{all}}=\{S_1,...,S_N \}$ of a DAG $\mathcal{G}$ are unique w.r.t.\ a set of variables $X_e$.
\end{proposition}

\noindent
Note that the separation into conditionally independent subsets is unique, as proven in Appendix~\ref{proof:unique}. An efficient way of finding the unique separation into conditionally independent subsets is described in Algorithm~\ref{algo_disjdecomp}, where each conditionally independent subset is found by expanding around a randomly selected starting node. The computational complexity of this algorithm is linear in the size of the largest conditionally independent subset.

By separating the relevant subgraph into conditionally independent subsets, the problem of calculating the marginal probability distribution $P(X_e)$ can be split into sub-problems that can be solved independently, thus reducing the overall complexity. 
\begin{proposition}[Marginalization in Subsets]\label{prop:subnet}
	Let $\mathcal{G}'$ be the relevant subgraph of a DAG $\mathcal{G}$ w.r.t.\ a set of nodes $e$, and let $p_{\mathcal{G}'}$ and $p_{\mathcal{G}}$ be the respective probability distributions that satisfy the Markov properties. Let $S_{\text{all}}=\{S_1,...,S_N \}$ be the conditionally independent subsets of the relevant subgraph. Then the marginal probability distribution $P(X_e)$ can be calculated by independently summing over the subsets $S_i$ according to
	\begin{equation}
	P(X_e)=P\big(X_{e'}\big)\prod _{S_i \in S_{\text{all}}} \sum_{X_{S_i}} P\Big(X_{S_i}, X_{e_i^{ch}} \mid X_{e_i^{mb}\setminus e_i^{ch}}\Big)
	\end{equation}
	where $mb(u)$ is the Markov blanket of node $u$, $e_i^{mb}= e \cap \{mb(u): u \in S_i \}$ and $e_i^{ch}= e \cap \{ch(u): u \in S_i \}$. $e' = e \setminus \{ e_i^{ch} \,  \forall i \}$ are evidence nodes that have either no parents or only evidence nodes as parents. 
\end{proposition}
By further exploiting the conditional independence of the subsets, it is possible to independently estimate some of the subsets by importance sampling (e.g.\ when exact inference is infeasible), while performing exact inference on the others. The full set of subsets would be split into the ones that are inferred exactly $S_{\text{exact}}$ and the ones that are estimated by means of importance sampling $S_{\text{approx}}$. 
\begin{proposition}[Estimator of Marginalization in Subsets]\label{prop:subnet2}
	Let $Q(X_{S_j})$ be the sampling distribution of the approximated subsets $S_j \in S_{\text{approx}}$ and let $S_j \in S_{\text{approx}}$ be the exactly inferred subsets. Then the marginal probability distribution $P(X_e)$ can be estimated according to
	\begin{equation}
	\begin{split}
	P(X_e)= P\big(X_{e'}\big)  &\prod _{S_i \in S_{\text{exact}}} \sum_{X_{S_i}} P\Big(X_{S_i}, X_{e_i^{ch}} \mid X_{e_i^{mb}\setminus e_i^{ch}}\Big) \\
	&\prod _{S_j \in S_{\text{approx}}} \mathbb{E}_{Q\big(X_{S_j}\big)}\left[\rule{0cm}{0.89cm} \frac{P\Big(X_{S_j} \mid X_{e_j^{mb}}\Big) P\Big(X_{e_j^{ch}} \mid X_{S_j}\Big)}{ Q\big(X_{S_j}\big)} \right]
	\end{split}
	\end{equation}
\end{proposition}
Applying the subgroup separation scheme leads to a variance that is lower or equal, as can be shown with the Rao-Blackwell theorem
\begin{equation}
Var \bigg[\frac{P(X_{S_i},X_{S_j})}{Q(X_{S_i},X_{S_j})}\bigg] \geq Var \bigg[\frac{P(X_{S_i})}{Q(X_{S_i})}\bigg] 
\end{equation}
where the probability distribution ${P(X_{S_i})=\sum_{X_{S_j}} P(X_{S_i},X_{S_j}) }$ and the importance distribution $Q(X_{S_i})=\sum_{X_{S_j}} Q\big(X_{S_i},X_{S_j}\big)$ (proof of inequality in \cite{maceachern1999}).

Since the subsets $S_{\text{exact}}$ which undergo exact inference have zero variance, the approximated subsets $S_{\text{approx}} \subseteq S_{\text{all}}$ introduce the variance of the subgroup separation scheme. By applying the Rao-Blackwell theorem, it follows that using the subgroup sampling scheme will lead to a lower or equal variance compared to a standard approximate inference method that is applied on the variables of all subsets 
\begin{equation}\label{equ:raoBlackwell}
Var \bigg[\frac{P(X_{S_{\text{all}}})}{Q(X_{S_{\text{all}}})}\bigg] \geq Var \bigg[\frac{P(X_{S_{\text{approx}}})}{Q(X_{S_{\text{approx}}}) }\bigg] , \text{ for } S_{\text{approx}} \subseteq S_{\text{all}}
\end{equation}
where equality is reached in the case of a fully connected DAG. In contrast to collapsed particles (also known as Rao-Blackwellisation and cutset sampling), where a subset of variables is also computed by using exact inference, subgroup separation comes with two major advantages. Firstly, finding the conditionally independent subsets is computationally cheap as it has a linear complexity and works on the DAG space (Algorithm~\ref{algo_disjdecomp}). Secondly, the computational complexity of the sampling process in our method does not grow exponentially with the number of variables computed with exact inference. This is in contrast to the collapsed particles approach, where the sampling space of the approximation grows exponentially with each variable computed by exact inference, since the approximation is not separated from the exact calculation. This issue is avoided in our separation, which allows for independent exact and approximate inference. 

\begin{algorithm}[h]
	\SetKwData{addedNode}{addedNode}
	\SetKwFunction{FindSubSets}{FindSubSets}\SetKwFunction{JunctionTree}{JunctionTree}
	\SetKwInOut{Input}{input}\SetKwInOut{Output}{output}
	\KwIn{A DAG $\mathcal{G} = (V, E)$ and evidence nodes $e$}
	\KwOut{Conditionally independent subsets $S_1,...,S_N$}
	\BlankLine
	
	%
	
	
	Non-evidence variables: $V'\leftarrow V\setminus e$
	
	$i \leftarrow 1$
	
	\Repeat{$V' = \emptyset$}{
		
		Pick a starting node: $V_{\text{added}} \leftarrow V'[1]$
		
		Initialize the subgroup: $S_i\leftarrow V_{\text{added}}$
		
		Starting from $V_{\text{added}}$, fill $S_i$:
		
		\Repeat{$V_{\text{added}} = \emptyset$}{
			
			$v \leftarrow V_{\text{added}}[1]$
			
			$V_{\text{added}} \leftarrow mb(v)\setminus \{ S_i \cup e \}$
			
			$S_i \leftarrow S_i \cup V_{\text{added}}$
			
			$V_{\text{added}}\leftarrow V_{\text{added}} \setminus v$
			
		}
		\vspace{0.2cm}
		Update remaining nodes: $V' \leftarrow V' \setminus S_i$
		
		$i\leftarrow i+1$
	}
	
	\caption{Finding Conditionally Independent Subsets (\texttt{FindSubSets})}\label{algo_disjdecomp}
\end{algorithm}

\begin{algorithm}[h]
	\SetKwData{DimLimit}{DimLimit}
	\SetKwFunction{FindSubSets}{FindSubSets}\SetKwFunction{JunctionTree}{JunctionTree}
	\SetKwFunction{ExactInference}{EI}\SetKwFunction{ApproximateInference}{AI}
	\SetKwFunction{JunctionTree}{JunctionTree}\SetKwFunction{StochSampl}{StochSampl}
	\SetKwInOut{Input}{input}\SetKwInOut{Output}{output}
	\KwIn{A Bayesian network $(\mathcal{G}, P)$ and evidence nodes $e$ with corresponding variables $(X_i)_{i\in e}$}
	\KwOut{Marginal probability distribution $P(X_e)\!\!\!\!$}
	\BlankLine
	
	%
	
	Determine relevant subgraph:
	$(\mathcal{G}', P')\leftarrow (\mathcal{G}, P) \mid _{e \cap \text{ancestors}(e)}$
	
	$S_1,..., S_N$ $\leftarrow$ \FindSubSets{$\mathcal{G}',e$}
	
	\For{$S_i\in S$}{
		\eIf{$ \mid S_i \mid $ $<$ $n_{\text{max}}$}{
			Use exact inference where possible:\\
			$P\Big( \! X_{e_i^{ch}}\! \mid \! X_{e_i^{mb}\setminus e_i^{ch}} \! \Big) \leftarrow$ \ExactInference\Big($\!X_{S_i}, X_{e_i^{mb}}\!$\Big)
		}{
			Use approximate inference where needed:\\
			$P\Big(\! X_{e_i^{ch}} \! \mid \! X_{e_i^{mb}\setminus e_i^{ch}} \! \Big) \leftarrow$ \ApproximateInference\Big($\! X_{S_i}, X_{e_i^{mb}}\!$\Big)
		}
	}
	Compute probability of remaining variables:
	$P\big(X_{e'}\big) \leftarrow \prod_{i \in e'} P\big(X_i \mid X_{pa(i)}\big)$ 
	
	Compute marginal probability distribution:
	$P(X_e) \leftarrow P\big(X_{e'}\big)\prod _{S_i \in S} P\Big(\!X_{e_i^{ch}} \! \! \mid  \! X_{e_i^{mb}\setminus e_i^{ch}}\!\!\Big)$
	
	\caption{Subgroup Separation (\texttt{SGS})}\label{algorithmSGS}
\end{algorithm}

\subsection{Subgroup Separation (SGS) Algorithm}\label{sec:SGS}

The structure of the subgroup separation (SGS) scheme is outlined in Algorithm~\ref{algorithmSGS} and can be described in four main parts. In the first part, the Bayesian network is reduced to its relevant subgraph and split into conditionally independent subsets $\{S_1,... S_N\}$ by using \texttt{FindSubSets} as described in Algorithm~\ref{algo_disjdecomp}. In the second and third parts, the individual subsets are processed either by exact (\texttt{EI}) or approximate inference (\texttt{AI}), depending on the size of the subset at hand. The fourth part concludes the calculation of the probability distribution of the left evidence nodes $e'$. 

As a consequence of the definition of nodes $e'$, their parents are a subset of the evidence nodes, i.e.\ $pa(v) \subset e,  \forall v \in e'$. This allows for the simple factorization
\begin{equation}
P\big(X_{e'}\big)=\prod_{v \in e'} P\Big(X_v \mid X_{pa(v)}\Big)
\end{equation}
While the choice of the conditionally independent subsets is unique (proof in Appendix~\ref{proof:unique}), their separation into exact and approximate inference subsets is specified by $n_{\text{max}}$. The hyperparameter $n_{\text{max}}$ tunes the ratio of exact and approximate inference in the marginalization, where a low $n_{\text{max}}$ results in a high fraction of approximate inference and a high $n_{\text{max}}$ results in a high fraction of exact inference. In our experiments, $n_{\text{max}}$ was set to $15$ to suit a broad range of parameters. More details on parameter selection and the algorithms complexity can be found in Appendix~\ref{sec:appAlgo} and \ref{sec:appalgocomple}, respectively. In our implementation, we use the junction tree algorithm for exact inference (\texttt{EI}) since it allows us to calculate the exact marginal probability distribution of each subset $P\Big(X_{e_i^{ch}} \mid X_{e_i^{mb}\setminus e_i^{ch}}\Big)$ at a low computational cost in low-dimensional subsets. Applying the dimensionality-reducing SGS scheme on the junction-tree algorithm reduces its computational cost, making exact inference possible for a
broader range of problems. On high-dimensional subsets that do not allow for exact inference, approximate inference is used for marginalization in the subset. As an approximate inference method (\texttt{AI}), we chose loopy belief propagation with subsequent importance sampling, where the importance sampling step guarantees that the approximation converges to the true marginal probability distribution. As its sub-calculations are performed by exact inference and importance sampling, which are unbiased methods \citep{koller2009}, the SGS scheme leads to an unbiased estimation of the marginal probability distribution. More details on the application of the junction-tree algorithm and sampling in the SGS scheme can be found in Appendix~\ref{appendix:junctionTree} and~\ref{appendix:importanceSampl}.

\section{Experiments}\label{sec:Experiments}

To assess the performance of our SGS scheme, we benchmarked it against standard inference schemes over a broad range of different Bayesian networks. To investigate structural dependencies, we simulated Bayesian networks of different dimensions, graph types, network densities, category sizes and fractions of evidence variables. For each setting, we simulated 100 Bayesian networks and selected the respective fraction of evidence variables at random. The CPTs $P_i$ were simulated from a uniform distribution, and then normalized to obtain discrete categorical variables. A more detailed description of the varied Bayesian network parameters can be found in Appendix~\ref{appendix:bench}.

\begin{figure}
	\includegraphics[scale=0.70]{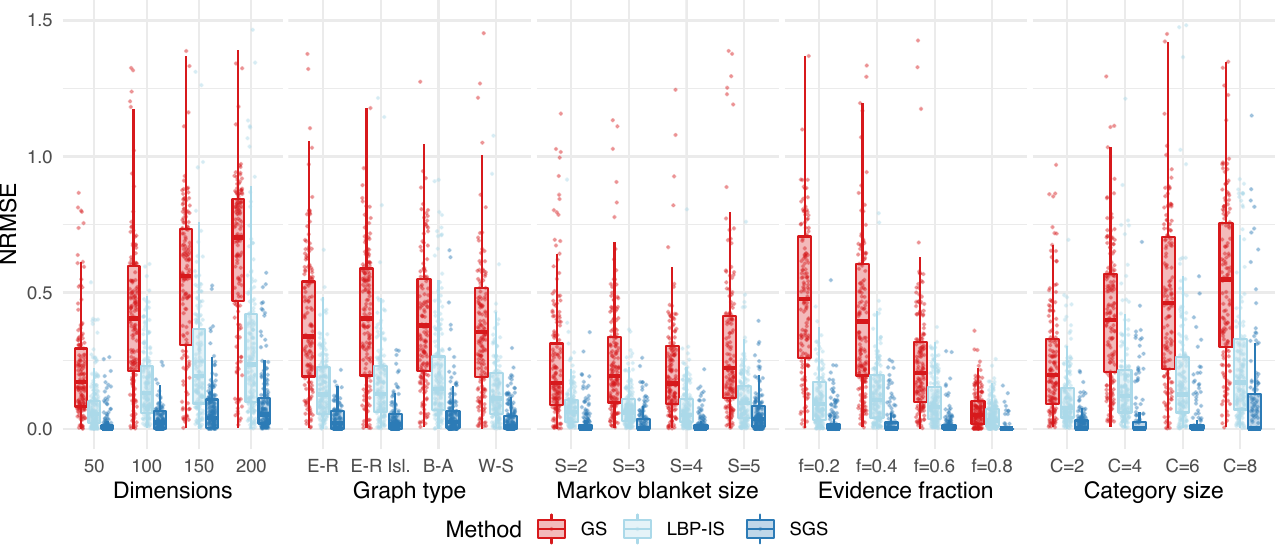}
	\caption{Boxplot of NRMSE of different sampling schemes for a fixed time ($t$=\( 0.2\)\,s). Corresponding plots over time are in Appendix~\ref{appendix:bench}. 
		GS: Gibbs sampling, LBP-IS: loopy belief propagation with subsequent importance weighting, SGS: subgroup separation;
		E-R: Erdős–Rényi, E-R Isl.: Erdős–Rényi island, B-A: Barabási–Albert, W-S: Watts-Strogatz.}
	\label{fig:all}
\end{figure}


Our SGS algorithm was benchmarked against Gibbs sampling and loopy belief propagation, which are prominent inference methods of MCMC sampling and variational inference, respectively. Since Gibbs sampling and loopy belief propagation do not directly yield the marginal probability distribution (as discussed in Section~\ref{sec:MessagePassing} and Section~\ref{sec:sampling}), subsequent importance sampling was used to estimate the marginal probability distribution, as proposed by \cite{yuan2006}; the methods will be referred to as GS and LBP-IS. 

Since the approximate inference step in the SGS scheme is performed with LBP-IS, the comparison of SGS and LBP-IS is of particular interest, as it reveals the gain in efficiency of the proposed calculation in conditionally independent subsets. 
While there are various other approximate inference schemes that could have been benchmarked, the relative gain in efficiency when using our SGS scheme is expected to be equivalent. 

To compare the efficiency of the different inference methods, the respective deviation of approximate and exact marginal probability distribution was calculated as a function of time.  
Note that while the Kullback-Leibler divergence is commonly used to compare approximate inference schemes, it is not suitable for comparing the inference of the marginal probability distribution of categorical variables as it requires the summation over all possible states of the evidence variables. This is not computationally feasible for the high-dimensional settings we consider. To account for the different sizes of the marginal probabilities of the simulated Bayesian networks, the normalized root mean squared error (NRMSE) was used for comparing the accuracy of each method according to
\begin{equation}
\text{NRMSE}=\sqrt{\frac{\sum_{i=1}^{n} (P(X_e) - \mathbb{E}_i[P(X_e)])^2}{n}}\cdot P(X_e)^{-1}
\end{equation}
where $\mathbb{E}_i[P(X_e)]$ is the i-th estimation of $P(X_e)$.

The \textbf{R} package \textbf{SubGroupSeparation}, including all used methods, reproducible benchmarks and applications, is available at \url{https://github.com/cbg-ethz/SubGroupSeparation}. The code is written in R and has core functions partially written in C. The \textbf{SubGroupSeparation} algorithm interfaces with Bayesian networks from the \textbf{bnstruct} package \citep{franzin2017bnstruct}, the \textbf{Bestie} package, and the \textbf{BiDAG} package \citep{suter2021bayesian,kuipers2022efficient}.




\subsection{Benchmark Results}\label{sec:benchres}

\begin{figure}%
	\captionsetup[subfigure]{labelformat=empty}
	\centering
	\subfloat[\centering  ]{{\includegraphics[scale=0.59]{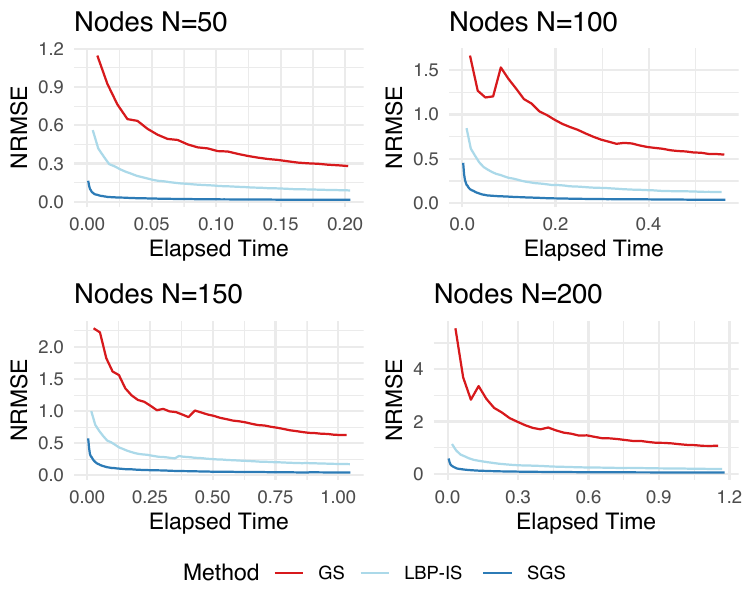} }}%
	\subfloat[\centering  ]{{\includegraphics[scale=0.59]{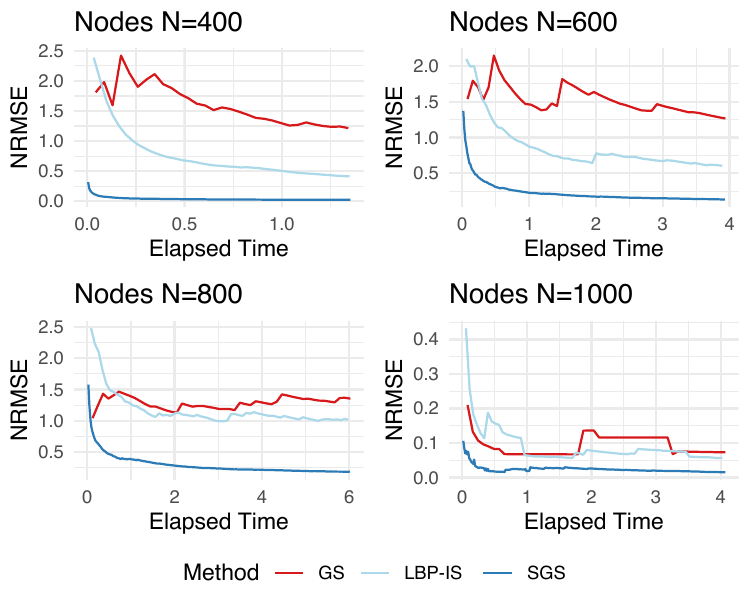} }}%
	\caption{NRMSE of different sampling schemes across different dimensions (lower is better). 
		GS: Gibbs sampling, LBP-IS: loopy belief propagation with subsequent importance weighting, SGS: subgroup separation.}%
	\label{fig:FigureDims}%
\end{figure}


A summarizing plot of our benchmarks across a broad range of Bayesian networks is given in Figure~\ref{fig:all}, showing the influence of different dimensions, graph types, network densities, category sizes and fractions of evidence variables. A more detailed analysis can be found in Appendix~\ref{appendix:bench}. 
Figure~\ref{fig:FigureDims} shows the average NRMSE over time for a broad range of dimensions, including high-dimensional settings up to a thousand variables.
The figures highlight, that our SGS scheme yields a high increase in accuracy across a broad range of Bayesian networks.
 
%


The combination of LBP and subsequent importance weighting (LBP-IS) appears to be more efficient than Gibbs sampling. 
Since SGS and LBP-IS use the same approximate inference method, the increase in efficiency in our SGS scheme can be ascribed to splitting the problem into conditionally independent subsets that are inferred independently. This confirms the theoretically expected lower variance of the SGS scheme (consequence of the Rao-Blackwell theorem in Equation~\ref{equ:raoBlackwell}).

In the limit of very high network densities, the efficiency gained with our SGS scheme is expected to decrease as no conditionally independent subgroups can be formed in a fully connected DAG.

\begin{figure}
	\centering
	\begin{subfigure}{.48\textwidth}
		\centering
		\includegraphics[width=.81\linewidth]{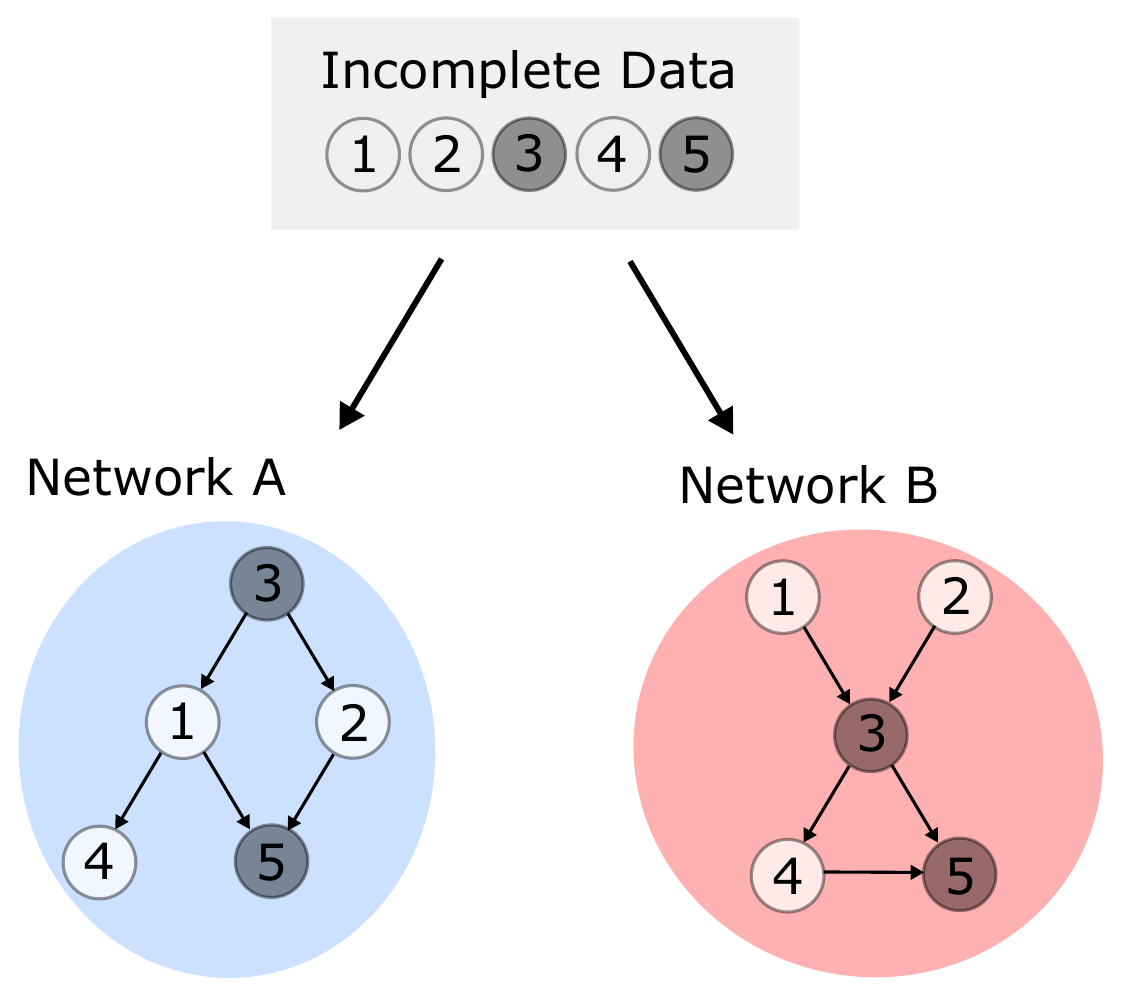}
		\caption{}
		\label{fig:test}
	\end{subfigure}%
	\begin{subfigure}{.5\textwidth}
		\centering
		\vspace{0.2cm}
		\input{figures/BarPlot_inference2.tex}
		\caption{}
		\label{fig:BarPlotMain}
	\end{subfigure}
	\caption{Classification of incomplete data against two Bayesian networks (a). Normalized probability of each cancer subtype (b). ccRCC: clear cell RCC, pRCC: papillary RCC.}
	\label{fig:notRel}
\end{figure}

\section{Application}\label{sec:appl}



As an immediate application of our SGS method, we demonstrate how the marginal probability distribution can be used to account for missing data in the classification of cancer subtypes. In particular, we employ Bayesian networks to identify the cancer subtype of a patient from the observed mutational pattern of their tumor. We considered a dataset taken from a small population study \citep{suh2020}, including mutational profiles of 25 patient samples diagnosed with renal cell carcinoma (RCC), which is the most common type of kidney cancer in adults \citep{hsieh2017}. The samples consist of the two most frequent histological subtypes, namely clear cell RCC (ccRCC) and papillary RCC (pRCC). To determine the cancer subtype of these patients, we classified their mutational profile against clusters learned on the TCGA database \citep{cancer2008}, as illustrated in Figure~\ref{fig:notRel}. On the TCGA data, we performed supervised Bayesian network modelling for the two subtypes ccRCC and pRCC, and learned a Bayesian network from the patient samples for each of the two subtypes. 
We considered the 70 most significantly mutated genes ($q < 0.1$) of the TCGA dataset, of which only 26 were also observed in the panel of the small population study. Given that only 26 of the 70 genes were observed in the small population study, marginalization was necessary to account for the missing variables in the classification. 
By using our SGS scheme, we calculated the probabilities $P(X_e|\mathcal{B}_i)$ of the patient samples $X_e$ to correspond to the Bayesian network models $\mathcal{B}_i$ with $i \in \{\text{ccRCC, pRCC}\}$. 
Figure~\ref{fig:notRel} shows the normalized probabilities of each cancer subtype for the 20 ccRCC (left) and the 5 pRCC (right) patient samples. 
The patient samples were assigned to the cancer subtype with the highest probability. We split the TCGA data into a training and a test dataset consisting of 80\% and 20\% of the samples, respectively to determine the network learning parameters. To quantify the impact of the marginal probability distribution, the whole analysis was repeated without marginalization by considering only the 26 mutual genes. 
More details on the analysis and a comprehensive comparison to the scenario without marginalization can be found in Appendix~\ref{sec:appappl}. 
While there are various other methods that can be used for classification, we focus here on showing how missing data can be handled in the application of Bayesian networks to biological data. 
%
By calculating the marginal probability distribution to account for the missing genes, we were able to assign \SI{76}{\percent} of the samples to the correct cancer type. In contrast, only \SI{68}{\percent} of the samples were assigned correctly when only mutual genes were considered in the analysis. 
The fact that not all samples were assigned to the correct cancer type also reflects the broad diversity of mutational profiles. 
Hence, a misclassification could indicate that the mutational profile is unusual for the cancer type at hand. This could imply that the tumor behaves differently than a typical tumor of the same type and can be exploited to further investigate potential for targeted treatment and survival prediction. 

%


\section{Conclusion}\label{sec:disc}

We have presented an efficient and scalable method for inferring the marginal probability distribution in Bayesian networks that splits the calculation into sub-calculations of lower dimension. Instead of applying one inference method for the whole calculation, our scheme applies exact and approximate inference on the lower-dimensional sub-calculations. 
The scheme we propose comes with an implementation that is compatible with several packages for Bayesian network structure learning and allows one to accurately handle missing data in Bayesian networks. 

Alongside theoretical justifications, we demonstrated empirically that our method highly increases accuracy. Over a broad range of different benchmarks, it reached more accurate estimates of the marginal probability distribution compared to state-of-the-art approximate inference methods. A limitation is that there is no gain in accuracy in the limit of fully connected DAGs, which is however unrealistic for Bayesian network applications, since fully connected DAGs are fully parametrised.
Our method is particularly efficient in high-dimensional sparse DAGs, which are expected in many real-world problems such as genomics and systems biology. 


The core idea of marginalization in lower-dimensional chunks of conditionally independent subsets is presented as a general framework that can be generalized to other exact and approximate inference schemes. 
Further, our proposed separation can be used to parallelize common inference algorithms, by running each sub-calculation in parallel. 


To date, missing data in Bayesian networks is commonly treated by imputing the missing variables \citep{scanagatta2018efficient,scutari2020bayesian, ruggieri2020hard}. 
However, imputing missing variables has two main error sources. First, imputation naturally introduces sampling error, and second, many variables will be imputed that have no impact on the inferred network structure, as discussed in Section~\ref{sec:reducedNet}. Both error sources can be avoided by directly calculating the marginal probability of the observed data using the method developed here. Hence, our method may constitute the basis for accurate structure learning of Bayesian networks with missing data in future work.

\acks{The authors are grateful to acknowledge funding support for this work from the two Cantons of Basel through project grant PMB-02-18 granted by the ETH Zurich.}


%
%
%
%
%
%

\newpage

\appendix
\section{Proof of Proposition~\ref{prop:relSub}} \label{proof:relSub}

\begin{proof}
	Let $B$ represent the set of irrelevant nodes of the DAG $\mathcal{G}$ w.r.t.\ the nodes $e$, then the nodes of the relevant subgraph $\mathcal{G}'$ are given by $V \setminus B$. Starting at Equation~\ref{equ:normConstDef}, we can separate the summation over the irrelevant nodes, which normalizes to one.
	\begin{equation*}
	\begin{split}
	P_{\mathcal{G}}(X_e) = & \sum_{X_{V\setminus e}} \prod_{v \in V} P\big(X_v \mid X_{pa(v)}\big) \\
	= & \sum_{ X_{\{ V \setminus B \} \setminus e}} \prod_{v \in V\setminus B} P\big(X_v \mid X_{pa(v)}\big) \sum_{X_B} \prod_{b \in B} P\big(X_b \mid X_{pa(b)}\big)\\
	= & \sum_{X_{V\setminus \{e \cup B\}}} \prod_{v \in V\setminus B} P\big(X_v \mid X_{pa(v)}\big) \\
	= & P_{\mathcal{G}'}(X_e)
	\end{split}
	\end{equation*}
	\noindent
	This allows us to rewrite the equation as the summation over the relevant subgraph. 
\end{proof}  

\section{Proof of Proposition~\ref{prop:unique}} \label{proof:unique}

\begin{proof}
	\noindent	
	We prove Proposition~\ref{prop:unique} by contradiction. Assume that conditionally independent subsets are not unique w.r.t.\ a set of nodes $e$. This implies that for a DAG $\mathcal{G}=(V, E)$ with nodes $e$, there exists a node which can be assigned to two different conditionally independent subsets $S_1, S_2$, i.e. 
	\begin{equation*}
	\exists \, v_1 \in V: \big( ( v_1 \in S_1 ) \land (v_1 \in S_2) \big) \land \big( S_1 \neq S_2 \big)
	\end{equation*}
	If $S_1$ and $S_2$ are different, then there must be a node, which is contained in only one set 
	\begin{equation*}
	\exists \, v_2 \in V:  \\
	\big( ( v_2 \in S_1 ) \land (v_2 \not\in S_2) \big) \lor \big( ( v_2 \not\in S_1 ) \land (v_2 \in S_2) \big)
	\end{equation*}
	We will show the contradiction for the case of ${( v_2 \in S_1 ) \land (v_2 \not\in S_2)}$, however it can analogously be shown for the alternative case of $(v_2 \not\in S_1 ) \land (v_2 \in S_2)$. 
	
	According to Definition~\ref{def:CIS}, $v_1$ and $v_2$ are d-connected since $S_1$ is a conditionally independent subset, i.e. 
	\begin{equation*}
	v_1, v_2 \in S_1 \Rightarrow v_1 \text{ and } v_2 \text{ are d-connected w.r.t.\ } e
	\end{equation*}
	At the same time Definition~\ref{def:CIS} yields the opposite for the conditionally independent subset $S_2$, i.e. 
	\begin{equation*}
	(v_1 \in S_2) \land (v_2 \not\in S_2) \Rightarrow v_1 \text{ and } v_2 \text{ are d-separated w.r.t.\ } e
	\end{equation*}
	The contradiction proves that conditionally independent subsets are unique w.r.t.\ a set of nodes $e$.
\end{proof}

\section{Proof of Proposition~\ref{prop:subnet}} \label{proof:subnet}

\begin{proof}
	Building on Proposition~\ref{prop:relSub}, we consider the nodes of the relevant subgraph $\mathcal{G}'$ and divide them into conditionally independent subsets $\{S_1,...,S_N \}$. We then separate the marginalization by considering the evidence nodes which are part of the Markov blanket of $S_i$, i.e.\ \mbox{$e_i^{mb}= e \cap \{mb(u): u \in S_i \}$}, and the evidence nodes which are children of the nodes of $S_i$, i.e.\ \mbox{$e_i^{ch}= e \cap \{ch(u): u \in S_i \}$}. 
	\begin{equation*}
	\begin{split}
	P(X_e)=& \frac{P(X_{V'},X_e)}{P(X_{V'} \mid X_e)} =\frac{P(X_{S_1},...,X_{S_N},X_e)}{P(X_{S_1},...,X_{S_N} \mid X_e)} =\frac{P(X_{S_1},...,X_{S_N},X_e)}{\prod_{i=1}^n P(X_{S_i} \mid X_e)}\\
	= & P\big(X_{e'}\big)\prod_{i=1}^nP\Big(X_{S_i} \mid X_{e_i^{pa}}\Big) \frac{P\Big(X_{e_i^{ch}} \mid X_{S_i},X_{e_i^{mb}\setminus e_i^{ch}}\Big)}{P\Big(X_{S_i} \mid X_{e_i^{mb}}\Big)}
	\end{split}
	\end{equation*}
	where $V'=V\setminus e$ and $e' = e \setminus \{ e_i^{ch}  \forall i \}$. Note that $P\Big(X_{S_i} \mid X_{e_i^{pa}}\Big)=P\Big(X_{S_i} \mid X_{e_i^{mb}\setminus e_i^{ch}}\Big)$ as $e_i^{ch}$ is a collider that d-separates $S_i$ from $e_i^{mb}\setminus \{e_i^{ch},e_i^{mb}\}$. This allows us to simplify
	\begin{equation*}
	\begin{split}
	P(X_e) & =  P\big(X_{e'}\big)\prod_{i=1}^N\frac{P\Big(X_{S_i},X_{e_i^{ch}} \mid X_{e_i^{mb}\setminus e_i^{ch}}\Big)}{P\Big(X_{S_i} \mid X_{e_i^{mb}}\Big)}\\
	& = P\big(X_{e'}\big)\prod_{i=1}^N\frac{P\Big(X_{e_i^{ch}} \mid X_{e_i^{mb}\setminus e_i^{ch}}\Big)P\Big(X_{S_i} \mid X_{e_i^{mb}}\Big)}{P\Big(X_{S_i} \mid X_{e_i^{mb}}\Big)} \\
	&= P\big(X_{e'}\big)\prod_{i=1}^N P\Big(X_{e_i^{ch}} \mid X_{e_i^{mb}\setminus e_i^{ch}}\Big)\\
	&=P\big(X_{e'}\big)\prod_{i=1}^N \sum_{X_{S_i}} P\Big(X_{S_i}, X_{e_i^{ch}} \mid X_{e_i^{mb}\setminus e_i^{ch}}\Big)
	\end{split}
	\end{equation*}
	\vspace{-0.5cm}\end{proof}

\section{Proof of Proposition~\ref{prop:subnet2}} \label{proof:subnet2}

We can prove Proposition~\ref{prop:subnet2} starting from Proposition~\ref{prop:subnet}.
\begin{proof}[Proof of Proposition~\ref{prop:subnet2}]
	Let $\mathcal{G}'$ be the relevant subgraph of a DAG $\mathcal{G}$ w.r.t.\ a set of nodes $e$ and let $p_{\mathcal{G}'}$ and $p_{\mathcal{G}}$ be the respective probability distributions that satisfy the Markov properties. Let \mbox{$S_{\text{all}}=\{S_1,...,S_N \}$} be the conditionally independent subsets of the relevant subgraph. According to Proposition~\ref{prop:subnet}, the marginal probability distribution $P(X_e)$ can be calculated as
	\begin{equation*}
	P(X_e)=P(X_{e'})\prod_{S_i \in S_{\text{all}}} \sum_{X_{S_i}} P\Big(X_{S_i}, X_{e_i^{ch}} \mid X_{e_i^{mb}\setminus e_i^{ch}}\Big)
	\end{equation*}
	where $e_i^{mb}= e \cap \{mb(u): u \in S_i \}$, $mb(u)$ is the Markov blanket of node $u$, $e_i^{ch}= e \cap \{ch(u): u \in S_i \}$ and $e' = e \setminus \{ e_i^{ch}  \forall i \}$.
	
	We can separate the product of the individual calculations of all subgroups ${S_{\text{all}}=S_{\text{exact}}\cup S_{\text{approx}}}$ into their approximate and exact parts. Expressing the approximate part as an expectation leads to the proposed equation. 
	\begin{equation*}
	\begin{split}
	P(X_e)= P(X_{e'}) & \prod_{S_j \in S_{\text{all}}} \sum_{X_{S_j}} P\Big(X_{S_j},X_{e_j^{ch}} \mid X_{e_j^{mb}\setminus e_j^{ch}}\Big)\\
	= P(X_{e'}) & \prod_{S_j \in S_{\text{exact}}} \sum_{X_{S_j}} P\Big(X_{S_j},X_{e_j^{ch}} \mid X_{e_j^{mb}\setminus e_j^{ch}}\Big)\\ 
	&\hspace{-1.3cm} \prod_{S_i \in S_{\text{approx}}} \sum_{X_{S_i}} Q(X_{S_i})\frac{P\Big(X_{S_i},X_{e_i^{ch}} \mid X_{e_i^{mb}\setminus e_i^{ch}}\Big)}{Q(X_{S_i})}\\
	= P(X_{e'}) & \prod_{S_j \in S_{\text{exact}}} \sum_{X_{S_j}} P\Big(X_{S_j},X_{e_j^{ch}} \mid X_{e_j^{mb}\setminus e_j^{ch}}\Big)\\ 
	&\hspace{-1.3cm} \prod_{S_i \in S_{\text{approx}}} \mathbb{E}_{Q(X_{S_i})} \left[\rule{0cm}{0.89cm} \frac{P\Big(X_{S_i},X_{e_i^{ch}} \mid X_{e_i^{mb}\setminus e_i^{ch}}\Big)}{Q(X_{S_i})}\right]
	\end{split}
	\end{equation*}
	
\end{proof}

\section{Parameter Selection}\label{sec:appAlgo}

The maximal dimension up to which exact inference is performed in the subsets $n_{\text{max}}$ is tuned to the sampling time at hand. For each subset, the computational cost of approximate inference must not exceed the computational cost of exact inference, i.e.\ a longer sampling time allows for higher $n_{\text{max}}$. The parameter $n_{\text{max}}$ is set as the dimension of a subset at which the average time of exact inference exceeds the average time of approximate inference. Note that $n_{\text{max}}$ depends on the given sampling time per number of variables in the relevant subgraph. In our experiments, $n_{\text{max}}$ was set to $15$ for fast inference in the broad parameter range described above.

\section{Algorithm Complexity}\label{sec:appalgocomple}

The computational complexity of Algorithm~\ref{algo_disjdecomp} for finding the conditionally independent subsets is linear in the size of the largest conditionally independent subset. 
Hence, finding the conditionally independent subsets is computationally cheap as it has a linear complexity and works on the DAG space. 
The problem of marginalization was proven to be NP-hard for categorical variables \citep{cooper1990}. Hence, while our proposed decomposition can reduce the computational complexity, the problem remains NP-hard in general. In the limit of fully connected DAGs, the computational complexity can not be reduced by means of our proposed separation. 

For exact inference of the conditional independent subsets, we use the junction-tree algorithm, which has a computational complexity that is exponential in the size of the largest clique. 
In our proposed decomposition, the cliques of the junction-tree algorithm can differ from the cliques of the plain junction-tree algorithm. The number of variables which are included in each clique is lower or equal compared to the cliques of the plain junction-tree algorithm. Since the complexity of the junction-tree algorithm grows exponentially with the number of variables in each clique, this leads to a lower or equal complexity of our method. We did not pre-compute the clique potentials since for different evidence variables, the clique potentials may need to be partially recomputed. If the calculation of the clique potentials for the whole Bayesian network is computationally infeasible (as is the case in high-dimensional settings), the fact that the potentials are not pre-computed can be beneficial since only a fraction of variables are relevant for the considered sets of evidence nodes. In our benchmarks, we recomputed the potentials for every query, though a speedup of our method could be reached in future by partially reusing the unchanged clique potentials.


\section{Additional Benchmark Information}\label{appendix:bench}


We generated Barabási–Albert graphs, Watts-Strogatz graphs and Erdős–Rényi island graphs, according to the corresponding options in the \textbf{pcalg} package \citep{kalisch2012}. Further parameters have been varied as follows: numbers of variables $n\in\{50, 100, 150, 200\}$, average Markov blanket sizes $S\in\{2, 3, 4,5\}$, category sizes $C\in\{2, 4, 6, 8\}$ and fraction of evidence variables ${f\in\{0.2,0.4,0.6,0.8\}}$. 

Figure~\ref{fig:benchmarkALL} shows the log NRMSE of the different inference schemes benchmarked over a range of different dimensions, graph types, category sizes, fractions of evidence variables and Markov blanket sizes. The figure shows, that the benchmark results are consistent over a broad range of Bayesian networks.

\begin{figure}
	\begin{subfigure}{.5\textwidth}
		\centering
		\includegraphics[width=0.95\linewidth]{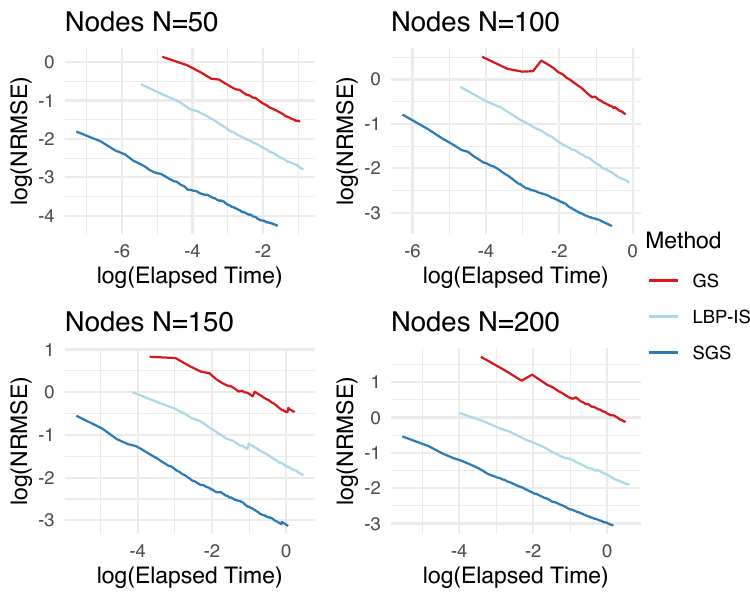}
		\caption{Varying dimensions}
		\label{fig:sfig1}
	\end{subfigure}%
	\begin{subfigure}{.5\textwidth}
		\centering
		\includegraphics[width=.95\linewidth]{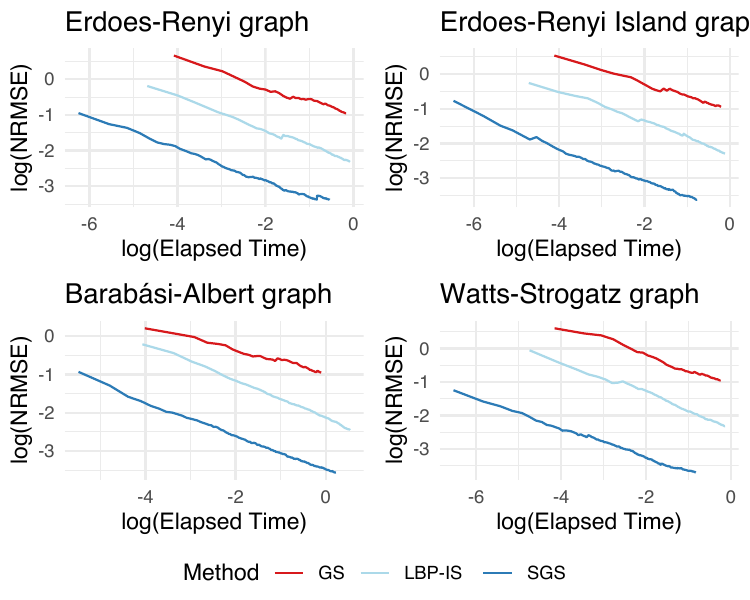}
		\caption{Varying graph types}
		\label{fig:sfig2}
	\end{subfigure}%
	\vspace{0.2cm}
	\begin{subfigure}{.5\textwidth}
		\centering
		\includegraphics[width=.95\linewidth]{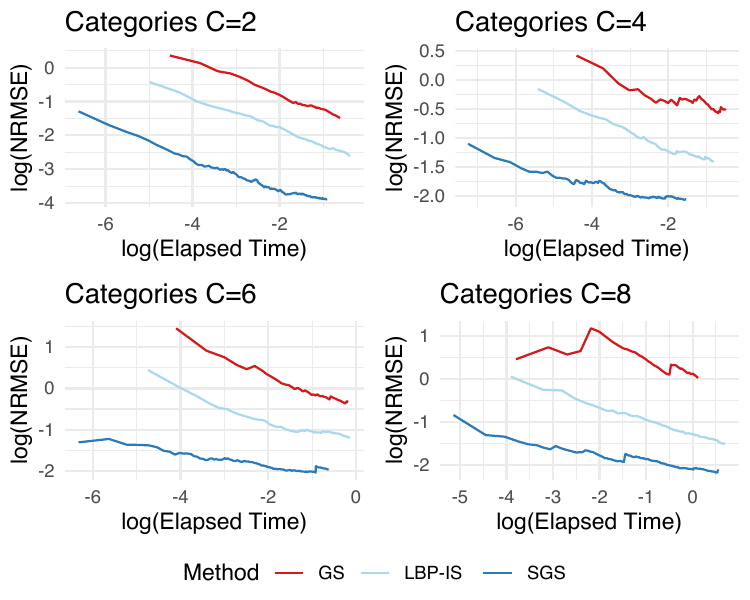}
		\caption{Varying category sizes}
		\label{fig:sfig2}
	\end{subfigure}%
	\begin{subfigure}{.5\textwidth}
		\centering
		\includegraphics[width=.95\linewidth]{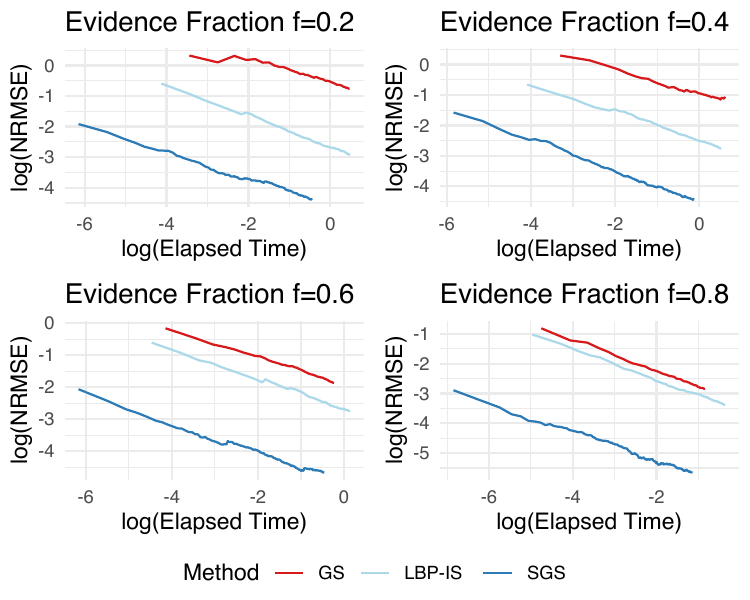}
		\caption{Varying fractions of evidence variables}
		\label{fig:sfig2}
	\end{subfigure}%
	
	\vspace{0.2cm}
	\centering
	\begin{subfigure}{.5\textwidth}
		\centering
		\includegraphics[width=.95\linewidth]{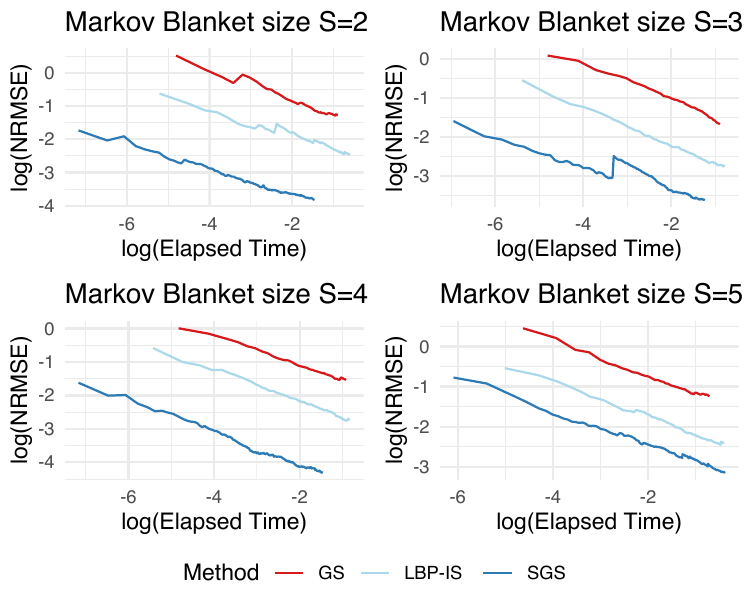}
		\caption{Varying Markov blanket sizes}
		\label{fig:sfig2}
	\end{subfigure}
	\caption{Log-log plot of the performance of different sampling schemes over a range of different dimensions, graph types, category sizes, fractions of evidence variables and Markov blanket sizes. 
		GS: Gibbs sampling, LBP-IS: loopy belief propagation with subsequent importance weighting, SGS: subgroup separation.}
	\label{fig:benchmarkALL}
\end{figure}

\begin{figure}
	\begin{subfigure}{.5\textwidth}
		\centering
		\includegraphics[width=0.95\linewidth]{figures/benchmark/FigureDims2.pdf}
		\caption{Varying dimensions}
		\label{fig:sfig12}
	\end{subfigure}%
	\begin{subfigure}{.5\textwidth}
		\centering
		\includegraphics[width=.95\linewidth]{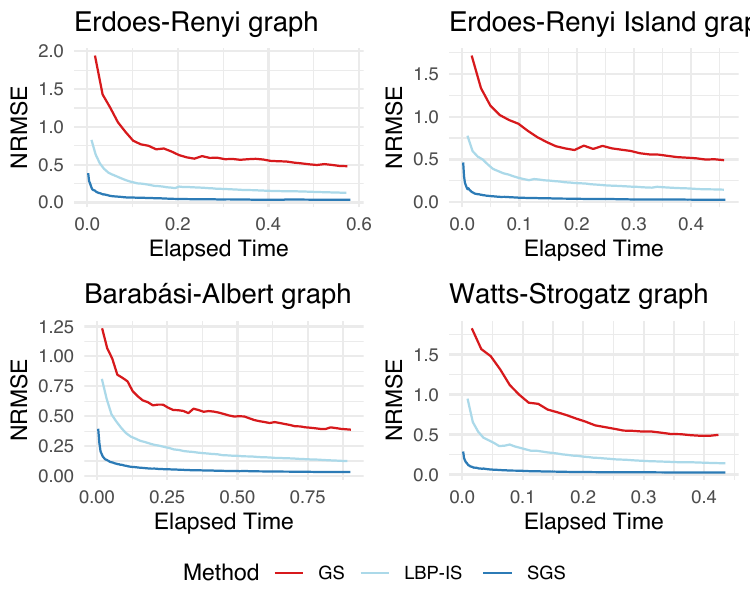}
		\caption{Varying graph types}
		\label{fig:sfig22}
	\end{subfigure}%
	\vspace{0.2cm}
	\begin{subfigure}{.5\textwidth}
		\centering
		\includegraphics[width=.95\linewidth]{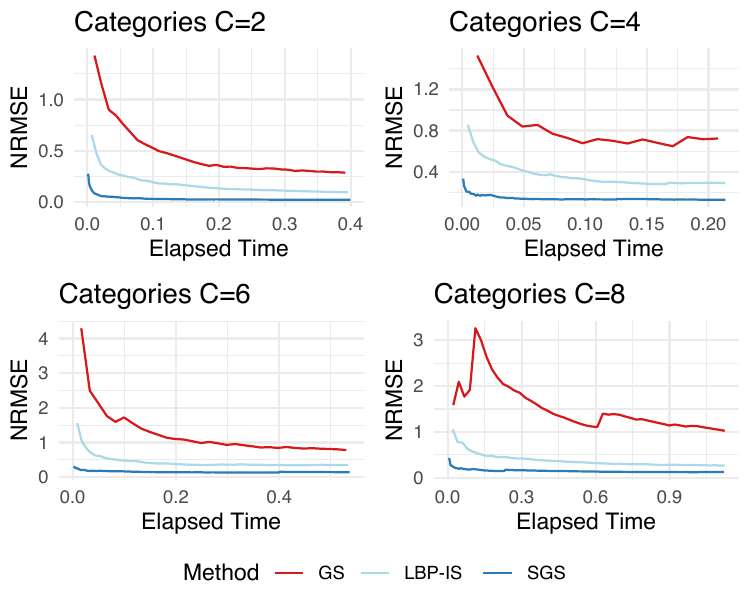}
		\caption{Varying category sizes}
		\label{fig:sfig32}
	\end{subfigure}%
	\begin{subfigure}{.5\textwidth}
		\centering
		\includegraphics[width=.95\linewidth]{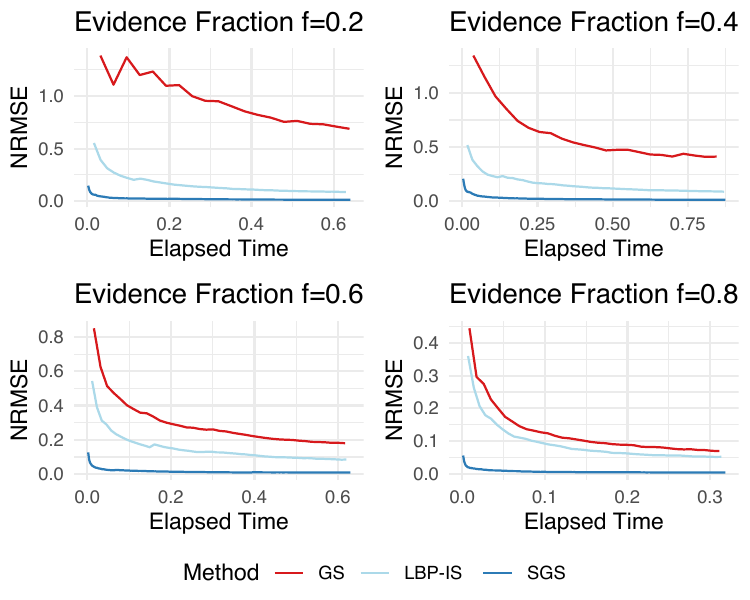}
		\caption{Varying fractions of evidence variables}
		\label{fig:sfig42}
	\end{subfigure}%
	
	\vspace{0.2cm}
	\centering
	\begin{subfigure}{.5\textwidth}
		\centering
		\includegraphics[width=.95\linewidth]{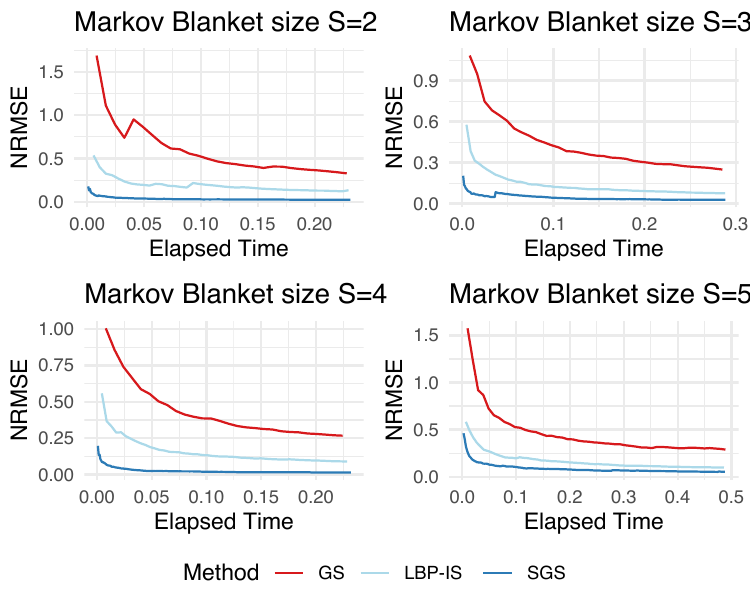}
		\caption{Varying Markov blanket sizes}
		\label{fig:sfig52}
	\end{subfigure}
	\caption{NRMSE over time of different sampling schemes over a range of different dimensions, graph types, category sizes, fractions of evidence variables and Markov blanket sizes. 
		GS: Gibbs sampling, LBP-IS: loopy belief propagation with subsequent importance weighting, SGS: subgroup separation.}
	\label{fig:benchmarkALL2}
\end{figure}

As the separation into conditionally independent subsets is highly dependent on the DAG structure, a separate benchmark included simulations of Barabási–Albert graphs, Watts-Strogatz graphs and Erdős–Rényi island graphs, generated according to the corresponding options in the \textbf{pcalg} package \citep{kalisch2012}. We find the SGS scheme to be an efficient method across all tested graph types.

Besides binary variables, different category sizes have been benchmarked in the range of $C\in\{2, 4, 6, 8\}$, showing that our SGS scheme increases the accuracy also for non-binary variables (see Figure~\ref{fig:benchmarkALL}\,(c)). This is expected as the conditionally independent subsets are independent of the category size in the Bayesian network. 

The evidence nodes $e$ were selected at random as a fraction $f\in\{0.2,0.4,0.6,0.8\}$ of the total number of variables. Besides the graph structure, the fraction of evidence $f$ plays a crucial role in forming conditionally independent subsets. Particularly high or low fractions of evidence increase the number of conditionally independent subsets, while decreasing the dimensionality within each subset. Consequently, the SGS scheme proves to be highly efficient in these settings. The sparsity of the DAGs has been varied such that the average Markov blanket size was $S\in\{2, 3, 4,5\}$, thus covering the range of sparse to dense networks. 

All computations were performed on performed on one CPU core of the AMD EPYC 7H12 processor (2.6 GHz nominal, 3.3 GHz peak) and 256 GB of DDR4 memory clocked at 3200 MHz. Code was partially used from the \textbf{bnstruct} package (GPL-3) \citep{franzin2017bnstruct}. 

\section{Application}\label{sec:appappl}

Bayesian networks can be employed to model the probabilistic relationships in biological data and can further be used to cluster mutational profiles to identify novel cancer subtypes \citep{kuipers2018}. However, biological data is often affected by missing information, which limits the application of Bayesian networks in this area.  As an immediate application of the proposed method, we demonstrate how the marginal probability distribution can be used to account for missing data in the classification of cancer subtypes. In particular, we employ Bayesian networks to identify the cancer subtype of a patient from the observed mutational pattern of their tumor. We considered a dataset taken from a Korean population study \citep{suh2020}, including mutational profiles of 25 patient samples diagnosed with renal cell carcinoma (RCC), which is the most common type of kidney cancer in adults \citep{hsieh2017}. The samples consist of the two most frequent histological subtypes, namely clear cell RCC (ccRCC) and papillary RCC (pRCC). While the real cancer subtype was known in our data and used in order to evaluate the accuracy of our prediction, the histological subtypes may not need to match molecular subtypes which can be employed to cluster patients based on their mutational profile \citep{kuipers2018}. 


\begin{figure*}[ht!]
	\centering
	\subfloat[\centering  Difference with/without Marginalization]
	{{\vspace{0.11cm}\input{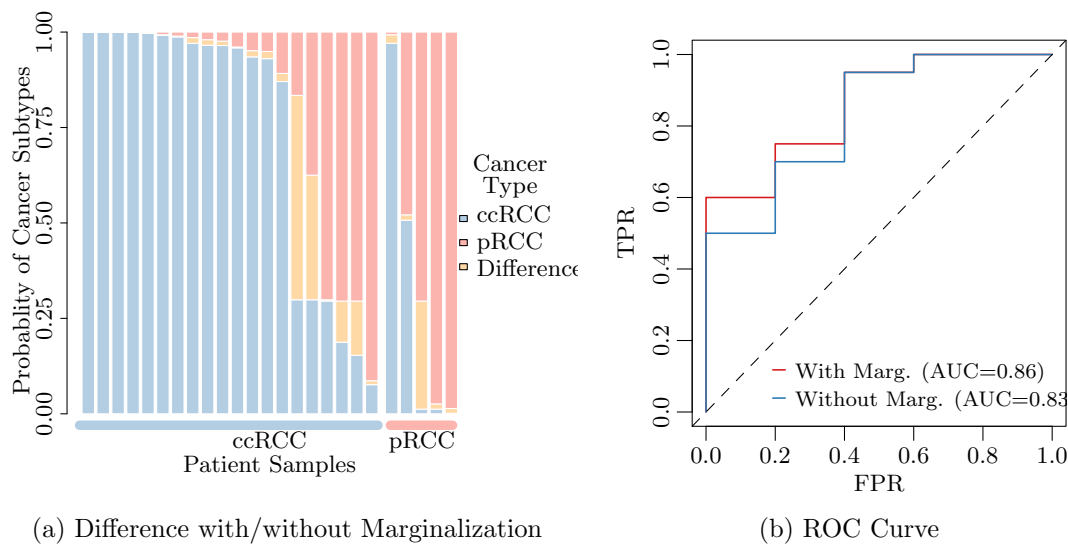} }}
	\subfloat[\centering ROC Curve]
	{{\vspace{-0.95cm}
\begin{tikzpicture}[x=0.945pt,y=0.89pt]
\definecolor{fillColor}{RGB}{255,255,255}
\path[use as bounding box,fill=fillColor,fill opacity=0.00] (0,0) rectangle (224.04,254.63);
\begin{scope}
\path[clip] ( 49.20, 61.20) rectangle (198.84,225.43);
\definecolor{drawColor}{RGB}{215,25,28}

\path[draw=drawColor,line width= 0.6pt,line join=round,line cap=round] (193.29,219.34) --
	(137.87,219.34) --
	(137.87,211.74) --
	(110.16,211.74) --
	(110.16,211.74) --
	(110.16,211.74) --
	(110.16,204.14) --
	(110.16,204.14) --
	(110.16,204.14) --
	(110.16,204.14) --
	(110.16,204.14) --
	(110.16,204.14) --
	(110.16,181.33) --
	(110.16,181.33) --
	(110.16,181.33) --
	(110.16,181.33) --
	(110.16,181.33) --
	(110.16,181.33) --
	(110.16,181.33) --
	( 82.45,181.33) --
	( 82.45,181.33) --
	( 82.45,181.33) --
	( 82.45,173.73) --
	( 82.45,173.73) --
	( 82.45,173.73) --
	( 82.45,173.73) --
	( 82.45,173.73) --
	( 82.45,173.73) --
	( 82.45,173.73) --
	( 82.45,173.73) --
	( 82.45,173.73) --
	( 82.45,173.73) --
	( 82.45,166.12) --
	( 82.45,166.12) --
	( 82.45,166.12) --
	( 82.45,166.12) --
	( 82.45,158.52) --
	( 54.74,158.52) --
	( 54.74, 67.28);
\end{scope}
\begin{scope}
\path[clip] (  0.00,  0.00) rectangle (224.04,274.63);
\definecolor{drawColor}{RGB}{0,0,0}

\path[draw=drawColor,line width= 0.4pt,line join=round,line cap=round] ( 54.74, 61.20) -- (193.29, 61.20);

\path[draw=drawColor,line width= 0.4pt,line join=round,line cap=round] ( 54.74, 61.20) -- ( 54.74, 55.20);

\path[draw=drawColor,line width= 0.4pt,line join=round,line cap=round] ( 82.45, 61.20) -- ( 82.45, 55.20);

\path[draw=drawColor,line width= 0.4pt,line join=round,line cap=round] (110.16, 61.20) -- (110.16, 55.20);

\path[draw=drawColor,line width= 0.4pt,line join=round,line cap=round] (137.87, 61.20) -- (137.87, 55.20);

\path[draw=drawColor,line width= 0.4pt,line join=round,line cap=round] (165.58, 61.20) -- (165.58, 55.20);

\path[draw=drawColor,line width= 0.4pt,line join=round,line cap=round] (193.29, 61.20) -- (193.29, 55.20);

\node[text=drawColor,anchor=base,inner sep=0pt, outer sep=0pt, scale=  1.00] at ( 54.74, 45.60) {\footnotesize 0.0};

\node[text=drawColor,anchor=base,inner sep=0pt, outer sep=0pt, scale=  1.00] at ( 82.45, 45.60) {\footnotesize 0.2};

\node[text=drawColor,anchor=base,inner sep=0pt, outer sep=0pt, scale=  1.00] at (110.16, 45.60) {\footnotesize 0.4};

\node[text=drawColor,anchor=base,inner sep=0pt, outer sep=0pt, scale=  1.00] at (137.87, 45.60) {\footnotesize 0.6};

\node[text=drawColor,anchor=base,inner sep=0pt, outer sep=0pt, scale=  1.00] at (165.58, 45.60) {\footnotesize 0.8};

\node[text=drawColor,anchor=base,inner sep=0pt, outer sep=0pt, scale=  1.00] at (193.29, 45.60) {\footnotesize 1.0};

\path[draw=drawColor,line width= 0.4pt,line join=round,line cap=round] ( 49.20, 67.28) -- ( 49.20,219.34);

\path[draw=drawColor,line width= 0.4pt,line join=round,line cap=round] ( 49.20, 67.28) -- ( 43.20, 67.28);

\path[draw=drawColor,line width= 0.4pt,line join=round,line cap=round] ( 49.20, 97.69) -- ( 43.20, 97.69);

\path[draw=drawColor,line width= 0.4pt,line join=round,line cap=round] ( 49.20,128.11) -- ( 43.20,128.11);

\path[draw=drawColor,line width= 0.4pt,line join=round,line cap=round] ( 49.20,158.52) -- ( 43.20,158.52);

\path[draw=drawColor,line width= 0.4pt,line join=round,line cap=round] ( 49.20,188.93) -- ( 43.20,188.93);

\path[draw=drawColor,line width= 0.4pt,line join=round,line cap=round] ( 49.20,219.34) -- ( 43.20,219.34);

\node[text=drawColor,rotate= 90.00,anchor=base,inner sep=0pt, outer sep=0pt, scale=  1.00] at ( 39.80, 67.28) {\footnotesize 0.0};

\node[text=drawColor,rotate= 90.00,anchor=base,inner sep=0pt, outer sep=0pt, scale=  1.00] at ( 39.80, 97.69) {\footnotesize 0.2};

\node[text=drawColor,rotate= 90.00,anchor=base,inner sep=0pt, outer sep=0pt, scale=  1.00] at ( 39.80,128.11) {\footnotesize 0.4};

\node[text=drawColor,rotate= 90.00,anchor=base,inner sep=0pt, outer sep=0pt, scale=  1.00] at ( 39.80,158.52) {\footnotesize 0.6};

\node[text=drawColor,rotate= 90.00,anchor=base,inner sep=0pt, outer sep=0pt, scale=  1.00] at ( 39.80,188.93) {\footnotesize 0.8};

\node[text=drawColor,rotate= 90.00,anchor=base,inner sep=0pt, outer sep=0pt, scale=  1.00] at ( 39.80,219.34) {\footnotesize 1.0};

\path[draw=drawColor,line width= 0.4pt,line join=round,line cap=round] ( 49.20, 61.20) --
	(198.84, 61.20) --
	(198.84,225.43) --
	( 49.20,225.43) --
	( 49.20, 61.20);
\end{scope}
\begin{scope}
\path[clip] (  0.00,  0.00) rectangle (224.04,274.63);
\definecolor{drawColor}{RGB}{0,0,0}

\node[text=drawColor,anchor=base,inner sep=0pt, outer sep=0pt, scale=  1.20] at (124.02,239.89) {};

\node[text=drawColor,anchor=base,inner sep=0pt, outer sep=0pt, scale=  1.00] at (124.02, 33.60) {\footnotesize FPR};

\node[text=drawColor,rotate= 90.00,anchor=base,inner sep=0pt, outer sep=0pt, scale=  1.00] at ( 25.80,143.31) {\footnotesize TPR};
\end{scope}
\begin{scope}
\path[clip] ( 49.20, 61.20) rectangle (198.84,225.43);
\definecolor{drawColor}{RGB}{44,123,182}

\path[draw=drawColor,line width= 0.6pt,line join=round,line cap=round] (193.29,219.34) --
	(137.87,219.34) --
	(137.87,211.74) --
	(110.16,211.74) --
	(110.16,211.74) --
	(110.16,211.74) --
	(110.16,204.14) --
	(110.16,204.14) --
	(110.16,196.53) --
	(110.16,196.53) --
	(110.16,196.53) --
	(110.16,196.53) --
	(110.16,173.73) --
	(110.16,173.73) --
	(110.16,173.73) --
	(110.16,173.73) --
	(110.16,173.73) --
	(110.16,173.73) --
	(110.16,173.73) --
	( 82.45,173.73) --
	( 82.45,173.73) --
	( 82.45,173.73) --
	( 82.45,173.73) --
	( 82.45,173.73) --
	( 82.45,173.73) --
	( 82.45,173.73) --
	( 82.45,173.73) --
	( 82.45,173.73) --
	( 82.45,173.73) --
	( 82.45,173.73) --
	( 82.45,173.73) --
	( 82.45,173.73) --
	( 82.45,173.73) --
	( 82.45,173.73) --
	( 82.45,166.12) --
	( 82.45,166.12) --
	( 82.45,143.31) --
	( 54.74,143.31) --
	( 54.74, 67.28);
\definecolor{drawColor}{RGB}{0,0,0}

\path[draw=drawColor,line width= 0.4pt,dash pattern=on 4pt off 4pt ,line join=round,line cap=round] ( 49.20, 61.20) -- (198.84,225.43);

\path[] ( 49.45, 97.20) rectangle (198.84, 61.20);
\definecolor{drawColor}{RGB}{215,25,28}

\path[draw=drawColor,line width= 0.6pt,line join=round,line cap=round] ( 81.45, 85.20) -- ( 86.45, 85.20);
\definecolor{drawColor}{RGB}{44,123,182}

\path[draw=drawColor,line width= 0.6pt,line join=round,line cap=round] ( 81.45, 73.20) -- ( 86.45, 73.20);
\definecolor{drawColor}{RGB}{0,0,0}

\node[text=drawColor,anchor=base west,inner sep=0pt, outer sep=0pt, scale=  1.00] at ( 90.45, 81.76) {\scriptsize With Marg. (AUC=0.86)};



\node[text=drawColor,anchor=base west,inner sep=0pt, outer sep=0pt, scale=  1.00] at ( 90.45, 69.76) {\scriptsize Without Marg. (AUC=0.83)};
\end{scope}
\end{tikzpicture}}}%
	\caption{Normalized probability of each cancer subtype if missing genes are excluded from the analysis (left) and ROC curve (right). The patient samples are sorted by probability and grouped by subtype.}%
	\label{fig:BarPlotRed}%
\end{figure*}

To determine the cancer subtype of the Korean population study, we classified their mutational profile against models learned on data from the TCGA database \citep{cancer2008}, as illustrated in Figure~\ref{fig:notRel}. On the TCGA data, we performed supervised Bayesian network modelling for the two subtypes ccRCC and pRCC, i.e., two individual Bayesian networks were learned from the patient samples for the two subtypes. The analysis can be split into two main parts:
\begin{enumerate}
	\item Supervised Bayesian network learning on the TCGA data (785 patient samples; 476 ccRCC, 282 pRCC)
	\item Classification of the Korean population study data (25 patient samples; 20 ccRCC, 5 pRCC)
\end{enumerate}
In the first step, for the TCGA dataset, we considered the 70 most significantly mutated genes ($q < 0.1$), of which only 26 were also observed in the panel of the Korean population study. The structure learning parameters (prior pseudo counts of the BDe score and edge penalization) were tuned by splitting the TCGA data into a training and a test dataset consisting of 80\% and 20\% of the samples, learning on the training data, and maximizing the accuracy of the classification of the test patient samples. 

In the second step, the patient samples of the Korean population study were classified into the two clusters learned from the TCGA data. Given that only 26 of the 70 genes from the learned Bayesian networks were observed in the Korean population study, the marginal probability distribution had to be calculated to account for the missing variables and classify the patient samples into each subtype.



In this application, the missing variables can also be simply excluded from the analysis since 26 of the genes are mutual between both panels. Note that this is often not possible if data is missing at random, which leads to a small fraction of mutually observed variables; in such cases, marginalization is the only option for handling the missing variables. To quantify the impact of the unobserved variables on the classification, the whole analysis was repeated considering only the 26 mutual genes in both steps. Hence, instead of considering the whole set of 70 genes from TCGA, only the mutual set of 26 genes were used to learn the Bayesian network models. In summary, the analysis has been performed in the following two ways:
\begin{enumerate}
	\item With marginalization: Learning on 70 genes (TCGA), classification 26 genes (Korean population study), marginalization over the remaining 44 genes
	\item Without marginalization: Learning on 26 genes (TCGA), classification 26 genes (Korean population study)
\end{enumerate}




\noindent
In the classification step, the probabilities $P(X_e|\mathcal{B}_i)$ of the patient samples $X_e$ to correspond to the Bayesian network models $\mathcal{B}_i$ with $i \in \{\text{ccRCC, pRCC}\}$ were calculated using the SGS scheme decribed in Algorithm~\ref{algorithmSGS}. Figure~\ref{fig:notRel} shows the normalized probabilities of each cancer subtype for the 20 ccRCC (left) and the 5 pRCC (right) patient samples. The patient samples were assigned to the cancer subtype with the highest probability. 

When we used marginalization to account for the missing genes in the classification of the patient samples from the Korean population study, \SI{76}{\percent} of the samples were assigned to the correct cancer type. In contrast, \SI{68}{\percent} of the samples were assigned to the correct cancer type if the missing genes were completely excluded from the analysis. As a reference,  in the classification of the TCGA test data with the full set of genes, \SI{83}{\percent} of the samples were assigned to the correct cancer type, which can be considered an upper bound. Comparing the assigned probabilities of the two analyses shows that the probability is closer to the correct cancer subtype if marginalization is used to account for the missing genes (Figure~\ref{fig:BarPlotRed}a), leading to a higher fraction of patient samples to be assigned to the correct cancer subtype. This is reflected in the ROC curve (Figure~\ref{fig:BarPlotRed}b) and the corresponding increase in AUC when we use the marginalisation enabled by our method.







\section{Inference Schemes in Subgroup Separation}

\subsection{Junction Tree in Subgroup Separation} \label{appendix:junctionTree}


After splitting the network into subgroups, the ones selected for exact inference $S_{\text{exact}}$ are processed with the junction tree algorithm. The basic idea of the junction tree algorithm is to cluster the nodes of the original network into cliques of nodes in order to remove the loops in the Bayesian network. The cliques form a rooted tree, called chordal graph, and are associated with corresponding clique potentials. 
To obtain the chordal graph, the DAG is moralized and subsequently triangulated to enable simple calculation of the marginals (further details on moralization and triangulation can be found, for example, in \cite{hojsgaard2012b}).

For each subgroup $S_i$, consider a subgraph $\mathcal{G}_{V_i}$ constructed on the nodes ${V_i=S_i \cup e_i^{mb}}$, and let $B_{V_i}$ be the corresponding Bayesian network. Let $(C_1,...,C_T)$ be the chordal graph of $\mathcal{G}_{V_i}$, then the clique potentials $\psi_{C_i}(X_{C_i})$ can be obtained by multiplying the CPTs of the individual nodes of the network. This summarizes the CPTs of the clique nodes $\prod_{v \in C_j} P(X_v \mid X_{pa(v)})$ into single potentials $\psi_{C_i}(X_{C_i})$, enabling us to write the probability distribution of $X_V$ as

\begin{equation}
P(X_V)=\prod_{1} \psi_{C_i}(X_{C_i})
\end{equation}

The evidence is incorporated by setting the values in the clique potentials $\psi_{C_i}(X_{C_i})$ to zero that are not identical to the values of the evidence nodes $e_i^{mb}$. In order to obtain the marginal probability distribution of each conditional independent subset ${P\Big(X_{e_i^{ch}} \mid X_{e_i^{mb}\setminus e_i^{ch}}\Big)}$ (rather than $P\left(X_{e_i^{mb}}\right)$), the values in the clique potentials that correspond to $e_i^{mb} \setminus e_i^{ch}$ are set to one. 

Let $C_r$ be a root of the chordal graph (this is not unique). We define messages that propagate from the leaves to the root of the chordal graph. The messages in \textit{sum-product message passing} are defined as

\begin{equation}
\delta_{i\rightarrow j} =\sum_{C_i\setminus S_{i,j}} \psi_{C_i} \cdot \prod_{k \in nb(i) \setminus j} \delta_{k\rightarrow i}
\end{equation}
where $S_{i,j}=C_i\cap C_j$ is called the sepset between $C_i$ and $C_j$ and $nb(i)$ are the neighbours of clique $C_i$. Prior to propagation, all messages are initialized to one.

The messages propagate upstream in the chordal graph to the root until every clique has passed its message apart from the root clique $C_r$. The \textit{belief} $\beta_{C_r}$ of the root clique is then given by

\begin{equation}
\beta_{C_r}=\psi_{C_r}\prod_{k \in nb(r)} \delta_{k\rightarrow r}
\end{equation}

The normalized belief gives the exact marginal probability distribution of clique $C_r$ given the evidence. Finally, the desired probability distribution of the subset is then given by the sum of the roots' belief $\beta_{C_r}$

\begin{equation}
P\Big(X_{e_i^{ch}} \mid X_{e_i^{mb}\setminus e_i^{ch}}\Big) = \sum_{C_r} \beta_{C_r}
\end{equation}

While the junction tree algorithm is an efficient method for exact inference, its runtime is exponential in the size of the largest clique. Hence, its application will not be possible in all subsets due to its computational cost.

\subsection{Importance Sampling in Subgroup Separation} \label{appendix:importanceSampl}

For the subsets $S_{\text{approx}}$ that are not feasible for the junction tree algorithm, an estimate needs to be found by using approximate inference. In order to find an unbiased estimator, importance sampling is used for approximating $P\Big(X_{e_i^{ch}} \mid X_{e_i^{mb}\setminus e_i^{ch}}\Big)$. The performance of importance sampling is largely dependent on the chosen importance distribution, which can be efficiently computed using variational inference \citep{yuan2006, li2017}. While loopy belief propagation is not suitable for finding the unbiased probability distribution of multiple variables given the evidence (as discussed in Section~\ref{sec:MessagePassing}), it has been shown to yield an efficient importance function that is cheap to compute \citep{yuan2006}.

\textit{Loopy} belief propagation represents an iterative application of the belief propagation algorithm; it is applied to graphs with loops. Belief propagation can be regarded as a special case of the junction tree algorithm, where messages are sent between individual nodes rather than cliques. Messages sent from children to parents are defined as
\begin{equation}
\delta_{a\rightarrow i} (X_i) = \sum_{X_a \setminus X_i} P_a(X_a) \prod_{j\in nb(a)\setminus i}  \delta_{j\rightarrow a} (X_j)
\end{equation}
whereas messages sent from parents to children are defined as
\begin{equation}
\delta_{i\rightarrow a} (X_i) = \prod_{c\in nb(i)\setminus a}  \delta_{c\rightarrow i} (X_i)
\end{equation}
where $nb(i)$ are the neighbouring nodes of node $i$ and all messages have one as an initial value.

The messages are propagated iteratively until their beliefs do not change significantly, commonly referred to as convergence. However, convergence is not generally guaranteed, as the beliefs can end up in oscillations resulting in endless repetitions. While there exist methods for handling the oscillations, they are not necessary in the estimation of an importance function, since a perfect solution is not required due to the subsequent correction with importance sampling. \cite{yuan2006} show that few iterations of LBP algorithm are sufficient for gaining an efficient importance distribution, defined by the beliefs of each node.

By drawing samples from the importance distribution $Q(X_{S_j})$ for each subset $S_j$, the probability distribution of the subset can be estimated according to


\begin{equation}
P\Big(X_{e_i^{ch}} \mid X_{e_i^{mb}\setminus e_i^{ch}}\Big)=\frac{1}{M} \sum_{m=1}^M  \frac{P\Big(X_{S_j}[m] \mid X_{e_j^{mb}\setminus e_i^{ch}}\Big) P\Big(X_{e_j^{ch}} \mid X_{S_j}[m]\Big)}{ Q(X_{S_j}[m])} 
\end{equation}

\vskip 0.2in
\bibliography{sample}

\begin{thebibliography}{46}
\providecommand{\natexlab}[1]{#1}
\providecommand{\url}[1]{\texttt{#1}}
\expandafter\ifx\csname urlstyle\endcsname\relax
  \providecommand{\doi}[1]{doi: #1}\else
  \providecommand{\doi}{doi: \begingroup \urlstyle{rm}\Url}\fi

\bibitem[Bidyuk and Dechter(2006)]{bidyuk2006}
B.~Bidyuk and R.~Dechter.
\newblock Cutset sampling with likelihood weighting.
\newblock \emph{Proceedings of the Twenty-Second Conference on Uncertainty in
  Artificial Intelligence}, pages 39--46, 2006.

\bibitem[Bidyuk and Dechter(2007)]{bidyuk2007}
B.~Bidyuk and R.~Dechter.
\newblock Cutset sampling for {Bayesian} networks.
\newblock \emph{Journal of Artificial Intelligence Research}, 28:\penalty0
  1--48, 2007.

\bibitem[Bühlmann et~al.(2014)Bühlmann, Kalisch, and Meier]{buhlmann2014}
P.~Bühlmann, M.~Kalisch, and L.~Meier.
\newblock High-dimensional statistics with a view toward applications in
  biology.
\newblock \emph{Annual Review of Statistics and Its Application}, 1:\penalty0
  255--278, 2014.

\bibitem[Casella and Robert(1996)]{casella1996a}
G.~Casella and C.~P. Robert.
\newblock Rao-{{Blackwellisation}} of sampling schemes.
\newblock \emph{Biometrika}, 83:\penalty0 81--94, 1996.

\bibitem[Constantinou et~al.(2021)Constantinou, Liu, Chobtham, Guo, and
  Kitson]{constantinou2021}
A.~C. Constantinou, Y.~Liu, K.~Chobtham, Z.~Guo, and N.~K. Kitson.
\newblock Large-scale empirical validation of {Bayesian} network structure
  learning algorithms with noisy data.
\newblock \emph{International Journal of Approximate Reasoning}, 131:\penalty0
  151--188, 2021.

\bibitem[Cooper(1990)]{cooper1990}
G.~F. Cooper.
\newblock The computational complexity of probabilistic inference using
  {B}ayesian belief networks.
\newblock \emph{Artificial Intelligence}, 42:\penalty0 393--405, 1990.

\bibitem[Franzin et~al.(2017)Franzin, Sambo, and
  Di~Camillo]{franzin2017bnstruct}
A.~Franzin, F.~Sambo, and B.~Di~Camillo.
\newblock bnstruct: An {R} package for {Bayesian} network structure learning in
  the presence of missing data.
\newblock \emph{Bioinformatics}, 33:\penalty0 1250--1252, 2017.

\bibitem[Friedman(2004)]{friedman2004}
N.~Friedman.
\newblock Inferring {{cellular networks using probabilistic graphical models}}.
\newblock \emph{Science}, 303:\penalty0 799--805, 2004.

\bibitem[Friedman et~al.(2000)Friedman, Linial, Nachman, and
  Pe'er]{friedman2000}
N.~Friedman, M.~Linial, I.~Nachman, and D.~Pe'er.
\newblock Using {{Bayesian networks}} to {{analyze expression data}}.
\newblock \emph{Journal of Computational Biology}, 7:\penalty0 601--620, 2000.

\bibitem[Gelman and Meng(1998)]{gelman1998}
A.~Gelman and X.-L. Meng.
\newblock Simulating normalizing constants: From importance sampling to bridge
  sampling to path sampling.
\newblock \emph{Statistical Science}, 13:\penalty0 163--185, 1998.

\bibitem[Gogate and Dechter(2008)]{gogate2008}
V.~Gogate and R.~Dechter.
\newblock {AND/OR} importance sampling.
\newblock \emph{Proceedings of the Twenty-Fourth Conference on Uncertainty in
  Artificial Intelligence}, pages 212--219, 2008.

\bibitem[Gogate and Dechter(2010)]{gogate2010}
V.~Gogate and R.~Dechter.
\newblock On combining graph-based variance reduction schemes.
\newblock \emph{Proceedings of the Thirteenth International Conference on
  Artificial Intelligence and Statistics}, pages 257--264, 2010.

\bibitem[Heinze-Deml et~al.(2018)Heinze-Deml, Maathuis, and
  Meinshausen]{heinze2018causal}
C.~Heinze-Deml, M.~H. Maathuis, and N.~Meinshausen.
\newblock Causal structure learning.
\newblock \emph{Annual Review of Statistics and Its Application}, 5:\penalty0
  371--391, 2018.

\bibitem[H{\o}jsgaard et~al.(2012)H{\o}jsgaard, Edwards, and
  Lauritzen]{hojsgaard2012b}
S.~H{\o}jsgaard, D.~Edwards, and S.~Lauritzen.
\newblock \emph{Graphical models with R}.
\newblock Springer Science \& Business Media, 2012.

\bibitem[Hsieh et~al.(2017)Hsieh, Purdue, Signoretti, Swanton, Albiges,
  Schmidinger, Heng, Larkin, and Ficarra]{hsieh2017}
J.~J. Hsieh, M.~P. Purdue, S.~Signoretti, C.~Swanton, L.~Albiges,
  M.~Schmidinger, D.~Y. Heng, J.~Larkin, and V.~Ficarra.
\newblock Renal cell carcinoma.
\newblock \emph{Nature Reviews Disease Primers}, 3:\penalty0 1--19, 2017.

\bibitem[Kalisch and B{\"u}hlman(2007)]{kalisch2007}
M.~Kalisch and P.~B{\"u}hlman.
\newblock Estimating high-dimensional directed acyclic graphs with the
  {PC}-algorithm.
\newblock \emph{Journal of Machine Learning Research}, 8\penalty0 (3), 2007.

\bibitem[Kalisch et~al.(2012)Kalisch, M{\"a}chler, Colombo, Maathuis, and
  B{\"u}hlmann]{kalisch2012}
M.~Kalisch, M.~M{\"a}chler, D.~Colombo, M.~H. Maathuis, and P.~B{\"u}hlmann.
\newblock Causal inference using graphical models with the {R} package pcalg.
\newblock \emph{Journal of statistical software}, 47:\penalty0 1--26, 2012.

\bibitem[Koller and Friedman(2009)]{koller2009}
D.~Koller and N.~Friedman.
\newblock \emph{Probabilistic {{graphical models}}: {{Principles}} and
  {{techniques}}}.
\newblock {MIT Press}, 2009.

\bibitem[Korb and Nicholson(2004)]{korb2010}
K.~B. Korb and A.~E. Nicholson.
\newblock \emph{Bayesian artificial intelligence}.
\newblock Chapman and Hall, 2004.

\bibitem[Kuipers et~al.(2018)Kuipers, Thurnherr, Moffa, Suter, Behr, Goosen,
  Christofori, and Beerenwinkel]{kuipers2018}
J.~Kuipers, T.~Thurnherr, G.~Moffa, P.~Suter, J.~Behr, R.~Goosen,
  G.~Christofori, and N.~Beerenwinkel.
\newblock Mutational interactions define novel cancer subgroups.
\newblock \emph{Nature Communications}, 9:\penalty0 4353, 2018.

\bibitem[Kuipers et~al.(2022)Kuipers, Suter, and Moffa]{kuipers2022efficient}
J.~Kuipers, P.~Suter, and G.~Moffa.
\newblock Efficient sampling and structure learning of {Bayesian} networks.
\newblock \emph{Journal of Computational and Graphical Statistics}, pages
  1--12, 2022.

\bibitem[Li et~al.(2017)Li, Turner, and Liu]{li2017}
Y.~Li, R.~E. Turner, and Q.~Liu.
\newblock Approximate {{inference}} with {{amortised MCMC}}.
\newblock \emph{arXiv preprint arXiv:1702.08343}, 2017.

\bibitem[Lin and Druzdzel(1997)]{lin1997}
Y.~Lin and M.~J. Druzdzel.
\newblock Computational advantages of relevance reasoning in bayesian belief
  networks.
\newblock \emph{Proceedings of the Thirteenth Conference on Uncertainty in
  Artificial Intelligence}, pages 342--350, 1997.

\bibitem[Luo et~al.(2020)Luo, Peng, and Ma]{luo2020causal}
Y.~Luo, J.~Peng, and J.~Ma.
\newblock When causal inference meets deep learning.
\newblock \emph{Nature Machine Intelligence}, 2:\penalty0 426--427, 2020.

\bibitem[Maathuis et~al.(2009)Maathuis, Kalisch, and
  B{\"u}hlmann]{maathuis2009}
M.~H. Maathuis, M.~Kalisch, and P.~B{\"u}hlmann.
\newblock Estimating high-dimensional intervention effects from observational
  data.
\newblock \emph{Annals of Statistics}, 37:\penalty0 3133--3164, 2009.

\bibitem[Maceachern et~al.(1999)Maceachern, Clyde, and Liu]{maceachern1999}
S.~N. Maceachern, M.~Clyde, and J.~S. Liu.
\newblock Sequential importance sampling for nonparametric {{Bayes}} models:
  {{The}} next generation.
\newblock \emph{Canadian Journal of Statistics}, 27:\penalty0 251--267, 1999.

\bibitem[Madsen and Jensen(1999)]{madsen1999}
A.~L. Madsen and F.~V. Jensen.
\newblock Lazy propagation: a junction tree inference algorithm based on lazy
  evaluation.
\newblock \emph{Artificial Intelligence}, 113\penalty0 (1-2):\penalty0
  203--245, 1999.

\bibitem[McLachlan et~al.(2020)McLachlan, Dube, Hitman, Fenton, and
  Kyrimi]{mclachlan2020}
S.~McLachlan, K.~Dube, G.~A. Hitman, N.~E. Fenton, and E.~Kyrimi.
\newblock Bayesian networks in healthcare: Distribution by medical condition.
\newblock \emph{Artificial Intelligence in Medicine}, 107:\penalty0 101912,
  2020.

\bibitem[Moffa et~al.(2017)Moffa, Catone, Kuipers, Kuipers, Freeman, Marwaha,
  Lennox, Broome, and Bebbington]{moffa2017}
G.~Moffa, G.~Catone, J.~Kuipers, E.~Kuipers, D.~Freeman, S.~Marwaha, B.~R.
  Lennox, M.~R. Broome, and P.~Bebbington.
\newblock Using directed acyclic graphs in epidemiological research in
  psychosis: an analysis of the role of bullying in psychosis.
\newblock \emph{Schizophrenia Bulletin}, 43:\penalty0 1273--1279, 2017.

\bibitem[Murphy(2012)]{murphy2012}
K.~P. Murphy.
\newblock \emph{Machine {{learning}}: {{A probabilistic perspective}}}.
\newblock {MIT Press}, 2012.

\bibitem[Murphy et~al.(1999)Murphy, Weiss, and Jordan]{murphy1999}
K.~P. Murphy, Y.~Weiss, and M.~I. Jordan.
\newblock Loopy belief propagation for approximate inference: an empirical
  study.
\newblock \emph{Proceedings of the Fifteenth Conference on Uncertainty in
  Artificial Intelligence}, pages 467--475, 1999.

\bibitem[Pearl(1988)]{pearl1988}
J.~Pearl.
\newblock \emph{Probabilistic reasoning in intelligent systems: {Networks} of
  plausible inference}.
\newblock Morgan Kaufmann Publishers Inc., 1988.

\bibitem[Pearl(1995)]{pearl1995}
J.~Pearl.
\newblock Causal diagrams for empirical research.
\newblock \emph{Biometrika}, 82:\penalty0 669--688, 1995.

\bibitem[Rios et~al.(2021)Rios, Moffa, and Kuipers]{rios2021}
F.~L. Rios, G.~Moffa, and J.~Kuipers.
\newblock Benchpress: a scalable and platform-independent workflow for
  benchmarking structure learning algorithms for graphical models.
\newblock \emph{arXiv preprint arXiv:2107.03863}, 2021.

\bibitem[Ruggieri et~al.(2020)Ruggieri, Stranieri, Stella, and
  Scutari]{ruggieri2020hard}
A.~Ruggieri, F.~Stranieri, F.~Stella, and M.~Scutari.
\newblock Hard and soft {EM} in {Bayesian} network learning from incomplete
  data.
\newblock \emph{Algorithms}, 13:\penalty0 329, 2020.

\bibitem[Scanagatta et~al.(2016)Scanagatta, Corani, De~Campos, and
  Zaffalon]{scanagatta2016}
M.~Scanagatta, G.~Corani, C.~P. De~Campos, and M.~Zaffalon.
\newblock Learning treewidth-bounded {Bayesian} networks with thousands of
  variables.
\newblock \emph{Advances in Neural Information Processing Systems}, 29, 2016.

\bibitem[Scanagatta et~al.(2018)Scanagatta, Corani, Zaffalon, Yoo, and
  Kang]{scanagatta2018efficient}
M.~Scanagatta, G.~Corani, M.~Zaffalon, J.~Yoo, and U.~Kang.
\newblock Efficient learning of bounded-treewidth {Bayesian} networks from
  complete and incomplete data sets.
\newblock \emph{International Journal of Approximate Reasoning}, 95:\penalty0
  152--166, 2018.

\bibitem[Sch{\"o}lkopf(2019)]{scholkopf2019causality}
B.~Sch{\"o}lkopf.
\newblock Causality for machine learning.
\newblock \emph{arXiv preprint arXiv:1911.10500}, 2019.

\bibitem[Scutari(2020)]{scutari2020bayesian}
M.~Scutari.
\newblock Bayesian network models for incomplete and dynamic data.
\newblock \emph{Statistica Neerlandica}, 74:\penalty0 397--419, 2020.

\bibitem[Spirtes et~al.(2000)Spirtes, Glymour, Scheines, and
  Heckerman]{spirtes2000}
P.~Spirtes, C.~N. Glymour, R.~Scheines, and D.~Heckerman.
\newblock \emph{Causation, prediction, and search}.
\newblock MIT Press, 2000.

\bibitem[Suh et~al.(2020)Suh, Jeong, Choi, Ku, Kim, Kim, and Kwak]{suh2020}
J.~Suh, C.~W. Jeong, S.~Choi, J.~H. Ku, H.~H. Kim, K.~Kim, and C.~Kwak.
\newblock Sharing the initial experience of pan-cancer panel analysis in
  high-risk renal cell carcinoma in the korean population.
\newblock \emph{BMC Urology}, 20:\penalty0 1--9, 2020.

\bibitem[Suter et~al.(2021)Suter, Kuipers, Moffa, and
  Beerenwinkel]{suter2021bayesian}
P.~Suter, J.~Kuipers, G.~Moffa, and N.~Beerenwinkel.
\newblock Bayesian structure learning and sampling of {Bayesian} networks with
  the {R} package {BiDAG}.
\newblock \emph{arXiv preprint arXiv:2105.00488}, 2021.

\bibitem[TCGA-Research-Network(2008)]{cancer2008}
TCGA-Research-Network.
\newblock Comprehensive genomic characterization defines human glioblastoma
  genes and core pathways.
\newblock \emph{Nature}, 455:\penalty0 1061, 2008.

\bibitem[Van De~Schoot et~al.(2017)Van De~Schoot, Winter, Ryan,
  Zondervan-Zwijnenburg, and Depaoli]{van2017}
R.~Van De~Schoot, S.~D. Winter, O.~Ryan, M.~Zondervan-Zwijnenburg, and
  S.~Depaoli.
\newblock A systematic review of {Bayesian} articles in psychology: The last 25
  years.
\newblock \emph{Psychological Methods}, 22:\penalty0 217, 2017.

\bibitem[Venugopal and Gogate(2013)]{venugopal2013}
D.~Venugopal and V.~Gogate.
\newblock Dynamic blocking and collapsing for {Gibbs} sampling.
\newblock \emph{Proceedings of the Twenty-Ninth Conference on Uncertainty in
  Artificial Intelligence}, pages 664--673, 2013.

\bibitem[Yuan and Druzdzel(2006)]{yuan2006}
C.~Yuan and M.~J. Druzdzel.
\newblock Importance sampling algorithms for {{Bayesian}} networks:
  {{Principles}} and performance.
\newblock \emph{Mathematical and Computer Modelling}, 43:\penalty0 1189--1207,
  2006.

\end{thebibliography}

\end{document}